%% file: main.tex
\documentclass{article}

\usepackage{iclr2027_conference,times}
\input{math_commands.tex}

\usepackage{amsmath}
\usepackage{amssymb}
\usepackage{amsthm}
\usepackage{booktabs}
\usepackage{multirow}
\usepackage{tabularx}
\usepackage{enumitem}
\usepackage{graphicx}
\usepackage{float}
\usepackage{wrapfig}
\usepackage{placeins}
\usepackage{pdflscape}
\usepackage{tikz}
\usetikzlibrary{arrows.meta,positioning,fit,calc,backgrounds,tikzmark}
\usepackage{hyperref}
\usepackage{url}

\setlist[itemize]{labelindent=0pt,leftmargin=0.85em,labelwidth=0.5em,labelsep=0.35em,itemindent=0pt,align=left}
\setlist[enumerate]{labelindent=0pt,leftmargin=1.25em,labelwidth=0.9em,labelsep=0.35em,itemindent=0pt,align=left}

\newtheorem{theorem}{Theorem}

\title{CasEm: A Cascade Architecture \\ for Long-Horizon Neural Emulation}

\author{%
\makebox[\dimexpr\textwidth-2\tabcolsep\relax][c]{%
\textbf{Zhaoyi Li$^{1,2,*}$ \quad Jingtao Ding$^{3,*,\dagger}$ \quad Shihua Li$^{1,\dagger}$}%
}%
}

\iclrfinalcopy

\usepackage{ragged2e}

\begin{document}

\maketitle
\lhead{}
\begingroup
\renewcommand{\thefootnote}{}
\begin{NoHyper}
\footnotemark
\footnotetext{%
$^1$ School of Automation, Southeast University; \quad
$^2$ Beijing Zhongguancun Academy; \quad
$^3$ Department of Earth System Science, Tsinghua University.
\par
$^*$ Equal contribution. \quad
$^\dagger$ Corresponding authors: Jingtao Ding (dingjt15@tsinghua.org.cn) and Shihua Li (lsh@seu.edu.cn).%
}
\end{NoHyper}
\endgroup
\setcounter{footnote}{0}

\input{sections/abstract}
\input{sections/introduction}
\input{sections/related_work}
\input{sections/method}
\input{sections/experiments}
\input{sections/limitations_conclusion}

\bibliography{references}
\bibliographystyle{iclr2027_conference}

\clearpage
\appendix
\raggedbottom
\input{sections/theory_appendix}
\clearpage
\input{sections/experimental_details_appendix}

\FloatBarrier
\input{sections/climate_results_appendix}

\end{document}

%% file: math_commands.tex
\usepackage{amsmath,amsfonts,bm}

\def\eqref#1{equation~\ref{#1}}

\def\1{\bm{1}}

\DeclareMathAlphabet{\mathsfit}{\encodingdefault}{\sfdefault}{m}{sl}
\SetMathAlphabet{\mathsfit}{bold}{\encodingdefault}{\sfdefault}{bx}{n}



%% file: sections/abstract.tex
\begin{abstract}
Autoregressive neural emulators can drift or diverge over long rollouts despite accurate short-term predictions. We introduce Cascaded Emulation (CasEm), a one-way rollout architecture that augments an existing full-state backbone with an independently evolving model of physically specified aggregates. Its forecasts guide corrections to full-state predictions, without feedback from the backbone to the aggregate model. Effective guidance requires aggregates that cover substantial backbone error, remain accurately predictable, and support useful full-state corrections. We derive a finite-horizon error bound that clarifies these three factors and use empirical diagnostics to guide subsystem selection.
Across four ODE/PDE benchmarks, CasEm reduces long-horizon rollout errors across diverse backbones and suppresses the trend toward error divergence in both diffusion tasks using Fourier neural operator backbones.
In global climate emulation, CasEm with a regional total-water subsystem reduces 10-year full-state time-mean error by $66.6\%$ and $46.3\%$ for frozen ACE and Spherical DYffusion backbones, respectively, while adding less than $3\%$ to inference time.
\end{abstract}

%% file: sections/introduction.tex
\section{Introduction}
\label{sec:introduction}

Data-driven neural emulators provide an efficient alternative to conventional numerical solvers for complex dynamical systems. By learning state transitions from high-dimensional trajectories, they can generate new spatiotemporal evolutions at substantially lower cost. This paradigm now spans weather forecasting
\citep{pathak2022fourcastnet,bi2023pangu,lam2023graphcast,chen2023fuxi,price2025gencast},
molecular dynamics \citep{musaelian2023allegro}, and a broad range of partial differential equations. Their computational efficiency enables large ensembles, extreme-event statistics, and long-term response studies.

However, maintaining accuracy and stability over long prediction horizons remains a central challenge. During autoregressive rollout, the model repeatedly consumes its own predictions, allowing small errors to accumulate and amplify, potentially distorting long-term statistics, violating physical constraints, or causing numerical divergence
\citep{bonavita2024limitations,chattopadhyay2024challenges,jiang2025hine}.
Low one-step error alone consequently does not guarantee a reliable long-horizon emulator.
For example, strong short-range weather forecasting skill does not guarantee stable climate emulation or accurate climate statistics
\citep{wattmeyer2023ace,ruhlingcachay2024sphericaldyffusion}.

Reliable long-horizon emulation has been pursued through rollout-aware training \citep{brandstetter2022mppde,list2025differentiability}, dynamical and statistical objectives \citep{li2022chaotic,jiang2023invariant}, structured or generative architectures \citep{li2021fno,tripura2022wno,ruhlingcachay2023dyffusion,hu2024wdno,rozet2025latent}, and refinement or multiscale prediction \citep{lippe2023pderefiner,morel2025multiscale,jiang2025hine}. Even with these advances, a full-state emulator must represent detailed evolution while keeping large-scale or collective behavior accurate over long rollouts. In some systems, physically specified aggregates have simpler dynamics than the full state: Fourier modes of the linear heat equation evolve independently \citep{trefethen2000spectral}, while center-of-mass position and velocity form a closed subsystem in the absence of external forces \citep{goldstein2011classical}. This motivates independently modeling selected aggregates with exactly or approximately closed dynamics, whose forecasts can guide the full-state rollout.

Motivated by this observation, we adapt the one-way structure of classical cascade interconnections \citep{khalil2002nonlinear} to introduce \textbf{Cascaded Emulation (CasEm)}: a one-way rollout architecture that retains an existing full-state backbone while adding an independently evolving model of physically specified aggregates upstream. A prescribed or learned coordination map uses its forecasts to correct the backbone's full-state predictions during rollout. Because the backbone does not feed back into the aggregate model, its accumulated errors cannot directly perturb subsequent aggregate forecasts. This division of labor can reduce long-horizon drift when the aggregate forecasts are accurate and their corrections are effective, while the backbone retains fine-scale dynamics.

Turning this cascade into an effective emulator requires choosing aggregates that cover a substantial portion of the backbone's rollout error, can be predicted accurately without full-state feedback, and produce useful full-state corrections. These requirements can conflict: adding aggregate variables may make more of the backbone's rollout error correctable, but less accurate subsystem forecasts can offset that gain. We derive a finite-horizon error bound that clarifies the roles of aggregate-direction error coverage, subsystem modelability, and full-state correctability, then use the corresponding empirical diagnostics to guide subsystem selection and correction design.

Across four ODE/PDE benchmarks and large-scale climate emulation, CasEm improves long-horizon accuracy with diverse backbones and suppresses explosive error growth in several benchmark configurations. In climate emulation, a subsystem built from regional averages of total water reduces 10-year full-state time-mean error by $66.6\%$ and $46.3\%$ with frozen ACE~\citep{wattmeyer2023ace} and Spherical DYffusion (S-DY)~\citep{ruhlingcachay2024sphericaldyffusion} backbones, respectively.

Our contributions are threefold: \textbf{(a) Cascaded rollout.} We introduce one-way coupling between an independently evolving aggregate subsystem and a recurrent full-state backbone, using aggregate forecasts to correct full-state predictions. \textbf{(b) Design guidance.} We derive a finite-horizon error bound and empirical diagnostics to guide subsystem selection and correction design. \textbf{(c) Cross-system validation.} We evaluate CasEm across ODE/PDE benchmarks and large-scale climate emulation with diverse backbones.

%% file: sections/related_work.tex
\section{Related Work}
\label{sec:related_work}

\textbf{Reduced-order and latent dynamics.} Autoencoder-based approaches learn low-dimensional coordinates and model their evolution using Koopman-based, sparse, or neural dynamical models \citep{lusch2018koopman,champion2019coordinates,linot2022reduced}. These approaches reconstruct full-state forecasts from the evolved coordinates, rather than correct predictions from a separately recurrent full-state backbone. Multiscale-LED alternates latent rollout with full-system simulation, feeding re-encoded states back into the latent recurrence \citep{vlachas2022led}. CasEm instead retains a recurrent full-state backbone and uses an independently evolving, physically specified aggregate subsystem to correct its predictions.

\textbf{Coarse-to-fine and multiscale methods.} Existing frameworks that pair autoregressive coarse emulators with downscaling evolve only the coarse state, where fine-scale fields are generated at each step without propagating them forward in time \citep{lupinjimenez2025fcds,perkins2025hiroace}. HINE conditions full-state prediction on a hierarchy of predicted coarse future states \citep{jiang2025hine}, while ReLIFT uses restrict--evolve--lift propagation \citep{bassi2026relift}. CasEm instead evolves both the aggregate subsystem and full-state backbone recurrently, with one-way guidance from the former to the latter. Further comparisons of temporal dependencies appear in Appendix~\ref{app:prediction_structures}.

\textbf{Climate emulation.} As an atmospheric application of long-horizon neural emulation, climate emulation seeks to extend the success of neural weather forecasting \citep{bi2023pangu,lam2023graphcast,price2025gencast} to stable multi-year simulations with realistic climate statistics. Representative efforts include ACE, ACE2, and S-DY \citep{wattmeyer2023ace,wattmeyer2025ace2,ruhlingcachay2024sphericaldyffusion}.
ACE uses a deterministic autoregressive SFNO for stable century-scale climate simulation. ACE2 adds explicit dry-air-mass and moisture conservation \citep{wattmeyer2023ace,wattmeyer2025ace2}. S-DY couples DYffusion with SFNO to produce stable century-scale probabilistic ensembles \citep{ruhlingcachay2024sphericaldyffusion}.
Building on these pretrained emulators, CasEm keeps the full-state backbones frozen and uses an independently evolving regional-water subsystem to reduce 10-year mean-field errors for both ACE and S-DY (Section~\ref{sec:climate_emulation}).

%% file: sections/method.tex
\section{CasEm: Macroscopic Trend First, Full-State Emulation Second}
\label{sec:method}
\input{figures/casem_architecture}

\subsection{Problem setup and cascaded rollout}

Consider a discrete-time dynamical system with optional forcing or control $u_t$,
\begin{equation*}
    x_{t+1}=\mathcal F(x_t,u_t),
    \qquad x_t\in\mathbb R^N.
\end{equation*}
Physical or domain knowledge specifies a small collection of aggregate variables that serve as the subsystem state. These variables typically capture large-scale or collective features of the full state. We use the default linear notation
\begin{equation*}
    z_t=Hx_t\in\mathbb R^n,
    \qquad n\ll N.
\end{equation*}
Nonlinear readouts are also supported. Examples include weighted statistics, prescribed Fourier or spherical-harmonic modes and bands, molecular structural coordinates, and conservation-related quantities. We focus on variables for which physical prior knowledge suggests exact or approximate closure,
\begin{equation*}
    z_{t+1}
    =G_t^\star(z_t,u_t)+r_t(x_t,u_t),
    \qquad
    \lVert r_t(x_t,u_t)\rVert_2
    \leq\varepsilon_{\mathrm{cl},t}(z_t,u_t).
\end{equation*}
Here, $r_t$ captures the influence of unresolved variables: closure is exact when $r_t\equiv0$ and useful approximate closure keeps this influence small on encountered states. Appendix~\ref{app:closure_q} formalizes this intuition and bounds its rollout effect.

\noindent
\begin{minipage}[t]{0.74\linewidth}
\vspace{0pt}
This closure structure provides a basis for independent subsystem modeling within CasEm. Standard full-state emulation approximates $\mathcal F$ with a single backbone $f_\theta$. CasEm retains this backbone and adds a low-dimensional model $g_\psi$ to learn the subsystem dynamics $G_t^\star$ from data, with the aggregate readout $H$ specified by prior knowledge. The subsystem rolls out independently, and a coordination map $\mathcal C_\phi$ (prescribed or learned) combines its forecast with the backbone proposal to produce the next full-state prediction. This coupling yields the following recurrence, initialized with $\widehat x_0=x_0$ and $\widehat z_0=Hx_0$:
\begin{equation}
\left\{
\begin{aligned}
    \widehat z_{t+1}
    &=g_\psi(\widehat z_t,u_t),\\
    \widehat x_{t+1}
    &=\mathcal C_\phi\!\left(
        f_\theta(\widehat x_t,u_t),
        \widehat z_{t+1}
    \right).
\end{aligned}
\right.
\label{eq:casem_rollout}
\end{equation}

These recurrences form a standard cascade interconnection, implemented here using an upstream-first discretization. Figure~\ref{fig:casem_block_diagram} depicts the CasEm recurrence in Eq.~(\ref{eq:casem_rollout}), while Figure~\ref{fig:casem_architecture} contrasts its independent subsystem evolution with standard autoregression and HINE's joint full/coarse prediction. Prediction hats and forcing inputs are omitted in Figure~\ref{fig:casem_architecture} for clarity. Appendix~\ref{app:cascade_background} provides the formal definition and background of cascade interconnections.
\end{minipage}\hfill
\begin{minipage}[t]{0.23\linewidth}
\vspace{0pt}
\centering
\includegraphics[width=1.13in]{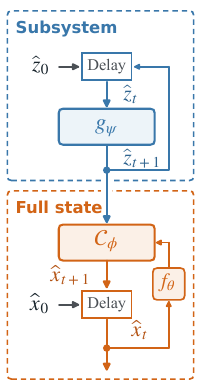}
\par\smallskip
\refstepcounter{figure}
\small Figure~\thefigure: CasEm block diagram ($u_t$ omitted).
\label{fig:casem_block_diagram}
\end{minipage}



\subsection{Mechanistic analysis}
\label{sec:theory_main}

All errors below accumulate squared deviations over the same evaluation trajectories and steps $1{:}T$, using a common normalized metric (Appendix~\ref{app:error_definitions}). Let $\mathcal E_x^{\mathrm{base}}$, $\mathcal E_x^{\mathrm{cas}}$, and $\mathcal E_x^{\mathrm{cas,orc}}$ denote the full-state errors of the uncorrected backbone, CasEm, and Oracle rollouts, respectively. Oracle uses true subsystem targets while keeping the backbone and correction map unchanged. For a fixed readout $H$, $\mathcal E_z^{\mathrm{base}}$ and $\mathcal E_z^{\mathrm{sub}}$ denote aggregate-variable errors from the base rollout and independent subsystem forecasts, respectively. For nonzero denominators, we define
\begin{equation*}
\boxed{
\rho
=\frac{\mathcal E_z^{\mathrm{base}}}{\mathcal E_x^{\mathrm{base}}},
\qquad
q
=\frac{\mathcal E_z^{\mathrm{sub}}}{\mathcal E_z^{\mathrm{base}}},
\qquad
\gamma
=\frac{\mathcal E_x^{\mathrm{base}}-\mathcal E_x^{\mathrm{cas,orc}}}
       {\mathcal E_z^{\mathrm{base}}}.
}
\end{equation*}
These quantities characterize three complementary aspects of CasEm:
\begingroup
\begin{itemize}

    \item \textbf{Aggregate-direction error coverage ($\rho$).} The fraction of the base model's full-state error captured by the selected aggregate directions, reflecting their importance to prediction error.
    \item \textbf{Modelability ($q$).} Subsystem prediction error relative to the base model's aggregate error, reflecting how accurately the subsystem can be modeled independently; $q<1$ indicates higher accuracy.
    \item \textbf{Correctability ($\gamma$).} Full-state error reduction per unit aggregate-direction error under exact subsystem targets. Positive $\gamma$ indicates beneficial correction; $\gamma=1$ matches the gain from fully compensating the original aggregate-direction error. Values $\gamma>1$ indicate additional gains in complementary directions, as aggregate correction can provide more dynamically consistent states for subsequent prediction through the coupled full-state dynamics.
\end{itemize}
\endgroup

The following theorem bounds CasEm's full-state error in terms of $\rho$, $q$, and $\gamma$.

\begin{theorem}[CasEm error bound]
\label{thm:main}
With $H$ normalized to have orthonormal rows, suppose that the correction makes the minimum change needed to satisfy the predicted aggregate constraint. If the corrected recurrence is well defined and has finite response to subsystem-target perturbations, summarized by a constant $B$, then CasEm's accumulated full-state error, normalized by the base-model error, satisfies
\begin{equation*}
\boxed{
\frac{\mathcal E_x^{\mathrm{cas}}}{\mathcal E_x^{\mathrm{base}}}
\leq
\left[
    \sqrt{1-\rho\gamma}
    +B\sqrt{\rho q}
\right]^2
+\rho q.
}
\end{equation*}
\end{theorem}

The normalization, correction rule, definition of $B$, and proof are detailed in Appendix~\ref{app:formal_theory}. The bound guides practical subsystem selection: it favors a larger ideal correction gain $\rho\gamma$ and a smaller normalized subsystem error $\rho q$. Candidate aggregate variables can therefore be screened sequentially for non-negligible error coverage $\rho$, good modelability (small $q$), and high correctability $\gamma$, following the pipeline in Figure~\ref{fig:subsystem_screening} (Appendix~\ref{app:subsystem_screening}).

%% file: figures/casem_architecture.tex
\begin{figure}[!t]
    \centering
    \begin{minipage}[c]{0.62\linewidth}
        \centering
        \includegraphics[width=\linewidth]{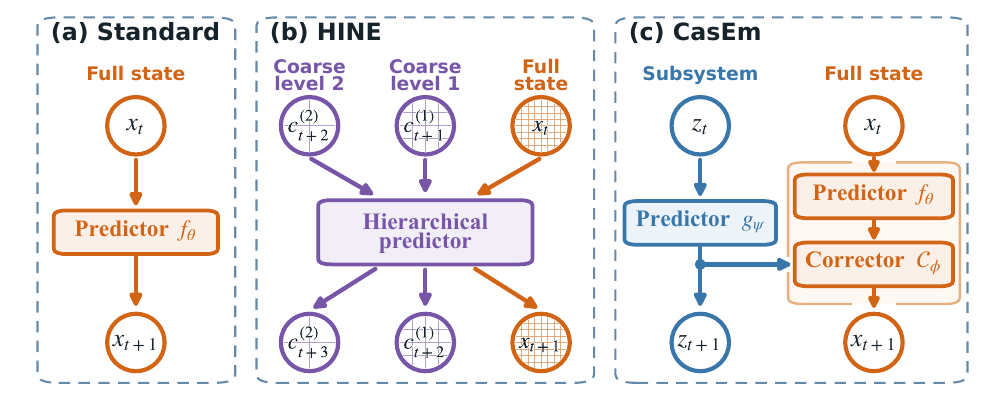}
    \end{minipage}\hfill%
    \begin{minipage}[c]{0.35\linewidth}
        \small
        \caption{One-step architectures: (a) standard autoregression; (b) HINE's joint full/coarse prediction, using prior-step predictions of spatially downsampled future states as coarse inputs; (c) CasEm's independent subsystem guiding full-state correction.}
        \label{fig:casem_architecture}
    \end{minipage}
\end{figure}

%% file: sections/experiments.tex
\section{ODE and PDE Benchmarks}
\label{sec:experiments}

We evaluate CasEm across four ODE/PDE systems and multiple backbones, alongside subsystem-selection and mechanism analyses. An additional central-force benchmark with a nonlinear aggregation map is reported in Appendix~\ref{app:central_force}. Full-state prediction accuracy is measured using normalized root mean square error (NRMSE).

\textbf{Tasks and datasets.} We begin with two diffusion-type PDEs on 1-dimensional periodic domains: the linear Heat equation and the nonlinear Fisher--KPP equation.
Heat describes heat conduction, whereas Fisher--KPP combines diffusion with nonlinear local growth to model reaction-front propagation.
We then extend the evaluation to 2-dimensional PDE dynamics
using the barotropic vorticity equation (BVE), a simplified model
of large-scale atmospheric flow on the sphere.
Finally, we move beyond field-based PDEs to a particle-based ODE
system: the Alanine benchmark, which simulates a 22-atom alanine dipeptide
at fixed particle number, volume, and energy (NVE).
Custom numerical solvers generate the Heat and Fisher--KPP trajectories, while the public simulation packages
SpeedyWeather.jl \citep{klower2024speedyweather} and OpenMM \citep{eastman2024openmm} generate the BVE and Alanine trajectories, respectively. Each prediction step advances the system by 0.05 simulation time units for Heat, 0.02 for Fisher--KPP, 3 hours for BVE, and 0.01 ps for Alanine. 

\textbf{Backbones and baselines.}
PDE backbones include residual convolutional neural networks (ResCNN), Fourier neural operators (FNO) for diffusion, and spherical Fourier neural operators (SFNO) \citep{bonev2023sfno} for BVE; Alanine uses task-adapted graph network simulator (GNS) \citep{sanchezgonzalez2020gns} and $E(n)$-equivariant graph neural network (EGNN)-velocity \citep{satorras2021egnn} backbones. Comparisons include the uncorrected Base and HINE. Unlike the PDE benchmarks, the Alanine ODE lacks a natural spatial-resolution hierarchy.
We therefore adapt HINE to condition full-state prediction on predicted future states of the same aggregate variables used by CasEm.
All baselines are trained with a rollout curriculum, starting with one-step prediction and progressively increasing the rollout length.
 Base and CasEm use identical pretrained backbone weights. We also include two diagnostic controls: Oracle and subsystem-only
rollouts.
Oracle replaces CasEm's predicted subsystem targets with their
ground-truth values.
Subsystem-only rollouts evolve only the subsystem state and
reconstruct the full state at each step using a task-specific
decoder, without a recurrent full-state model.
Decoder construction is detailed in
Appendix~\ref{app:subsystem_only_decoders}.

\textbf{Subsystem choice.} The Heat and Fisher--KPP subsystems use low-frequency Fourier coefficients, which describe the spatial mean and large-scale structure. This pair contrasts exact closure in Heat, where modes evolve independently, with approximate closure in Fisher--KPP, where nonlinear reactions couple modes. The BVE subsystem consists of the 3 real degree-1 spherical-harmonic coefficients, representing planetary-scale rotational motion. For Alanine, the subsystem comprises center-of-mass position and velocity, describing overall molecular translation decoupled from internal motion. Subsystem definitions and implementation details are provided in Appendix~\ref{app:controlled_details}.

\subsection{Overall Performance}
\label{sec:benchmark_results}
\label{sec:results}

\begin{figure}[!t]
\centering
\includegraphics[width=\linewidth]{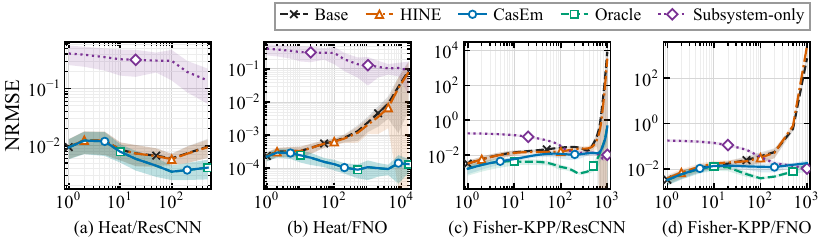}
\caption{Full-state NRMSE versus rollout steps. Lines show means; shaded bands indicate $\pm$ SD.}
\label{fig:diffusion_main}
\vspace{2mm}
\includegraphics[width=\linewidth]{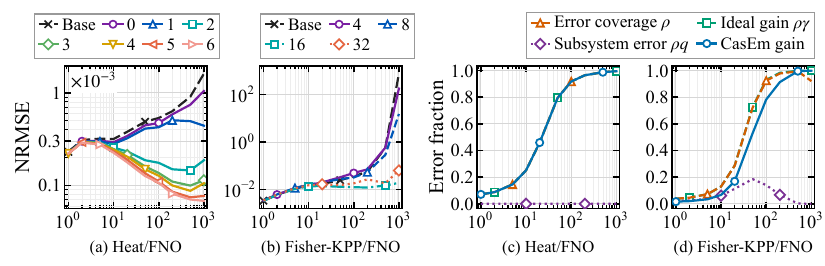}
\caption{Subsystem-size ablations and mechanism diagnostics for Heat and Fisher--KPP with FNO. (a,b) Full-state NRMSE for retained Fourier modes $0{:}K$; legends give $K$. (c,d) Cumulative diagnostics normalized by cumulative Base full-state squared error. Horizontal axes: rollout steps.}
\label{fig:subsystem_analysis}
\end{figure}

\begin{table*}[!t]
\centering
\begin{minipage}[t]{0.67\linewidth}
\vspace{0pt}
\fontsize{7.2}{8.0}\selectfont
\setlength{\tabcolsep}{1.2pt}
\renewcommand{\arraystretch}{0.91}
\begingroup
\medmuskip=0mu
\begin{tabular}{@{}c@{\hspace{1pt}}c@{\hspace{2pt}}lcccc@{}}
\toprule
& & Method & One-step & Intermediate & Long & Extended \\
\midrule
\multirow{9}{*}{\rotatebox[origin=c]{90}{BVE}} & \multirow{4}{*}{\rotatebox[origin=c]{90}{ResCNN}} & Base & $0.0117\pm0.00171$ & $0.703\pm0.163$ & $14.9\pm2.54$ & $65.7\pm7.03$ \\
& & HINE & $0.0113\pm0.00167$ & $0.677\pm0.157$ & $13.4\pm2.29$ & $63.5\pm6.76$ \\
& & CasEm & $\mathbf{0.00943\pm0.00128}$ & $\mathbf{0.488\pm0.112}$ & $\mathbf{2.17\pm0.320}$ & $\mathbf{4.96\pm3.17}$ \\
& & Oracle & $0.00943\pm0.00128$ & $0.483\pm0.115$ & $2.14\pm0.376$ & $5.19\pm3.22$ \\
\cmidrule(lr){2-7}
& \multirow{4}{*}{\rotatebox[origin=c]{90}{SFNO}} & Base & $0.0250\pm0.00312$ & $0.159\pm0.0401$ & $1.36\pm1.61$ & $4.98\pm3.99$ \\
& & HINE & $0.0241\pm0.00309$ & $0.161\pm0.0403$ & $1.39\pm1.58$ & $5.53\pm4.05$ \\
& & CasEm & $\mathbf{0.0164\pm0.00313}$ & $\mathbf{0.154\pm0.0355}$ & $\mathbf{0.773\pm0.102}$ & $\mathbf{1.72\pm2.68}$ \\
& & Oracle & $0.0164\pm0.00313$ & $0.154\pm0.0355$ & $0.772\pm0.102$ & $1.73\pm2.72$ \\
\cmidrule(lr){2-7}
& & SO & $0.113\pm0.0175$ & $1.11\pm0.0834$ & $1.11\pm0.106$ & $1.09\pm0.0772$ \\
\specialrule{\heavyrulewidth}{\aboverulesep}{\belowrulesep}
\multirow{9}{*}{\rotatebox[origin=c]{90}{Alanine}} & \multirow{4}{*}{\rotatebox[origin=c]{90}{GNS}} & Base & $0.200\pm0.0369$ & $2.18\pm0.568$ & $6.95\pm2.39$ & $32.9\pm13.4$ \\
& & HINE & $0.382\pm0.105$ & $1.889\pm0.469$
& $5.692\pm2.403$ & $95.680\pm18.833$ \\
& & CasEm & $0.199\pm0.0365$ & $\mathbf{1.26\pm0.205}$ & $\mathbf{1.28\pm0.203}$ & $\mathbf{1.27\pm0.207}$ \\
& & Oracle & $0.199\pm0.0365$ & $1.23\pm0.203$ & $1.28\pm0.158$ & $1.24\pm0.181$ \\
\cmidrule(lr){2-7}
& \multirow{4}{*}{\rotatebox[origin=c]{90}{EGNN}} & Base & $0.424\pm0.0826$ & $1.81\pm0.529$ & $5.45\pm2.60$ & $26.6\pm13.4$ \\
& & HINE & $0.444\pm0.092$ & $1.369\pm0.194$ & $6.082\pm1.855$ & $29.443\pm14.520$ \\
& & CasEm & $0.423\pm0.0822$ & $\mathbf{1.20\pm0.175}$ & $\mathbf{1.25\pm0.160}$ & $\mathbf{1.21\pm0.117}$ \\
& & Oracle & $0.423\pm0.0822$ & $1.19\pm0.155$ & $1.26\pm0.158$ & $1.23\pm0.127$ \\
\cmidrule(lr){2-7}
& & SO & $0.870\pm0.112$ & $1.03\pm0.107$ & $1.07\pm0.188$ & $3.36\pm1.89$ \\
\bottomrule
\end{tabular}
\endgroup
\end{minipage}\hfill%
\begin{minipage}[t]{0.31\linewidth}
\vspace{0pt}
\small

\caption{Full-state endpoint NRMSE for BVE and alanine dipeptide (mean $\pm$ sample SD across trajectories). The four horizons correspond to 1/\allowbreak 50/\allowbreak 500/\allowbreak 1000 steps for BVE and 1/\allowbreak 500/\allowbreak 2000/\allowbreak 10000 steps for Alanine. Within each backbone and horizon, bold highlights the lowest mean among Base, HINE, and CasEm unless differences are marginal. Oracle and backbone-independent subsystem-only (SO) rollouts serve as diagnostic references.}
\label{tab:other_controlled_results}
\end{minipage}
\end{table*}

Figure~\ref{fig:diffusion_main} and Table~\ref{tab:other_controlled_results} summarize the full-state rollout errors. Experimental details and additional results are provided in Appendix~\ref{app:controlled_details}.
The main findings are as follows:
\textbf{(a) CasEm improves long-horizon prediction across tasks and backbones.}
CasEm consistently reduces long-horizon prediction errors relative
to the corresponding Base models across the tested systems,
backbone architectures, and subsystem constructions.
Notably, for both diffusion benchmarks with FNO backbones
in Fig.~\ref{fig:diffusion_main}, CasEm suppresses the trend
toward error divergence observed in Base and HINE.
With ResCNN backbones, CasEm slows this trend even when it does
not fully prevent it.
CasEm also generally exhibits smaller sample standard deviations (SD)
than the baselines across different experimental settings,
particularly over longer rollouts, indicating more stable
prediction performance.
Backbone design remains important: FNO achieves lower long-horizon
errors than ResCNN, even when its one-step Base error is
slightly larger.
CasEm further improves both backbones, showing that subsystem
guidance complements backbone-level improvements.
\textbf{(b) Full-state fidelity and subsystem stability are complementary.}
Full-state emulators generally provide higher short-horizon accuracy,
but some Base rollouts develop explosive error growth over longer
horizons and eventually perform worse than subsystem-only predictions.
Although subsystem-only models retain only selected aspects of the
full state and cannot fully represent detailed dynamics, their
simplified low-dimensional dynamics support highly stable
long-term predictions.
This contrast supports the design rationale of CasEm: combining
the full-state emulator's ability to represent detailed dynamics
with the subsystem model's long-term stability.
\textbf{(c) HINE provides limited gains over Base in these benchmarks.} Despite its reported gains on multiscale Navier--Stokes turbulence \citep{jiang2025hine}, HINE does not consistently outperform Base on our PDE benchmarks. Our diagnostics highlight persistent aggregate drift as a major error source that hierarchical future-state conditioning alone may not prevent. On the molecular ODE, task-adapted HINE sometimes improves over Base at intermediate horizons but incurs large errors over longer rollouts. Its unnecessary feedback from full-state predictions into closed aggregate dynamics may amplify accumulating errors.

\subsection{Subsystem Choice and Mechanism Analysis}
\label{sec:subsystem_choice}

Using Heat and Fisher--KPP with FNO backbones, we examine how subsystem size and frequency content affect CasEm and interpret the gains through $\rho$, $q$, and $\gamma$. Figure~\ref{fig:subsystem_analysis}(a,b) summarizes the size ablations, and panels (c,d) provide mechanism diagnostics. Complete size-ablation results 
and frequency-band comparisons are in Appendix~\ref{app:frequency_band_ablation}. The main findings are as follows:
\textbf{(a) Larger subsystems do not necessarily yield better predictions.} Heat benefits from larger subsystems, with diminishing absolute gains beyond $K=4$. In Fisher--KPP, increasing $K$ from 16 to 32 worsens CasEm despite improving Oracle results (Appendix Table~\ref{tab:subsystem_size_fisher}). This suggests that the modeling difficulty of larger subsystems can offset their broader error coverage. \textbf{(b) Long-horizon gains depend on both correctability and modelability.} Aggregate directions capture a growing share $\rho$ of accumulated Base error over most of the rollout. The ideal gain $\rho\gamma$ closely tracks $\rho$, indicating effective correctability ($\gamma\approx1$). Heat's exactly closed subsystem has negligible normalized error $\rho q$, so CasEm nearly realizes the ideal gain. In Fisher--KPP, larger errors in the approximately closed subsystem accompany a visible actual--ideal gap, especially at intermediate horizons.

\section{Climate Emulation}
\label{sec:climate_emulation}

We next evaluate CasEm in large-scale global climate emulation, involving substantially higher-dimensional states and extensive trajectory data.

\textbf{Dataset and baselines.} The dataset comprises 11 simulations of 10 years each from the global atmospheric model FV3GFS \citep{zhou2019fv3gfs}. Outputs are sampled every 6 hours and regridded to a $1^\circ$ Gaussian grid. The evaluated state contains 34 prognostic fields, including total water at 8 vertical levels. Comparisons pair the deterministic ACE and probabilistic S-DY backbones with their respective CasEm variants, keeping both pretrained backbones frozen. Rollouts are evaluated on NVIDIA A100 GPUs; the inference times in Table~\ref{tab:climate_runtime} are measured on a single GPU.

\textbf{Subsystem choice.} The water subsystem comprises regional averages of total water at each vertical level. Its dynamics combine a fitted seasonal--trend background with linear autoregressive predictions of low-frequency residuals. An energy-based subsystem is also evaluated. Construction details and energy-subsystem results are provided in Appendices~\ref{app:climate_details} and~\ref{app:climate_energy}.

\textbf{Evaluation metrics.} Spatiotemporal (ST) NRMSE measures area-weighted trajectory errors, whereas time-mean (TM) NRMSE measures persistent climate biases through area-weighted errors in 10-year mean fields. This emphasis on mean-climate fidelity follows S-DY and established CMIP/AMIP evaluation practice \citep{ruhlingcachay2024sphericaldyffusion,lee2024pmp}. Both metrics are reported for the full state and total-water fields. Ensemble-mean NRMSE and the fair continuous ranked probability score (CRPS) additionally evaluate probabilistic time-mean predictions; definitions and normalization are provided in Appendix~\ref{app:climate_ensemble}.

\subsection{Overall Performance}

\suppressfloats[t]

\begin{table}[!htbp]
\centering
\begin{minipage}[t]{0.73\linewidth}
\vspace{0pt}
\centering
\small
\setlength{\tabcolsep}{1.5pt}
\begingroup
\medmuskip=0mu
\begin{tabular}{@{}clcccc@{}}
\toprule
& & \multicolumn{2}{c}{Full state}
  & \multicolumn{2}{c}{Total water} \\
\cmidrule(lr){3-4}\cmidrule(lr){5-6}
\multicolumn{6}{@{}l}{\textit{Rollout NRMSE}} \\
& Method & TM & ST & TM & ST \\
\midrule
\multirow{2}{*}{\rotatebox[origin=c]{90}{ACE}}
& Base  & $0.1580$ & $0.9409$ & $0.4691$ & $1.3832$ \\
& CasEm & $0.0527$(66.6\%$\uparrow$) & $0.7913$(15.9\%$\uparrow$) & $0.0840$(82.1\%$\uparrow$) & $0.8671$(37.3\%$\uparrow$) \\[-1pt]
\cmidrule(lr){2-6}
\multirow{3}{*}{\rotatebox[origin=c]{90}{S-DY}}
& Base
& $0.0785\pm0.0028$
& $0.8160\pm0.0029$
& $0.2134\pm0.0095$
& $0.9517\pm0.0097$ \\
& CasEm
& $0.0421\pm0.0016$
& $0.7929\pm0.0019$
& $0.0653\pm0.0014$
& $0.8553\pm0.0025$ \\
&
& (46.3\% $\uparrow$)
& (2.8\% $\uparrow$)
& (69.4\% $\uparrow$)
& (10.1\% $\uparrow$) \\[-1pt]
\midrule
\multicolumn{6}{@{}l}{\textit{Ensemble scores (normalized)}} \\
& Method
& \shortstack{Ensemble-mean\\NRMSE}
& Fair CRPS
& \shortstack{Ensemble-mean\\NRMSE}
& Fair CRPS \\
\midrule
\multirow{3}{*}{\rotatebox[origin=c]{90}{S-DY}}
& Base  & 0.0721 & 0.0546 & 0.2073 & 0.1782 \\
& CasEm & 0.0358 & 0.0206 & 0.0595 & 0.0367 \\
&
& (50.3\% $\uparrow$)
& (62.2\% $\uparrow$)
& (71.3\% $\uparrow$)
& (79.4\% $\uparrow$) \\
\bottomrule
\end{tabular}
\endgroup
\end{minipage}\hfill%
\begin{minipage}[t]{0.25\linewidth}
\vspace{0pt}
\small

\caption{10-year comparison of ACE, S-DY, and their water-guided CasEm variants. Both backbones are frozen. ACE: one rollout; S-DY: 25 paired members (mean $\pm$ sample SD). Ensemble scores use 10-year mean fields. Parentheses show CasEm's relative gains over the corresponding Base ($\uparrow$).}
\label{tab:climate_results}
\end{minipage}
\end{table}

\begin{wraptable}{r}{0.34\textwidth}

\small

\caption{10-year inference time (min); overhead in parentheses.}
\label{tab:climate_runtime}
\centering
\setlength{\tabcolsep}{1.5pt}
\begin{tabular}{@{}lcc@{}}
\toprule
Time & ACE & S-DY \\
\midrule
Base  & 12.14 & 37.03 \\
CasEm & 12.45\,{\scriptsize $(+2.6\%)$}
      & 37.54\,{\scriptsize $(+1.4\%)$} \\
\bottomrule
\end{tabular}

\end{wraptable}

Tables~\ref{tab:climate_results}--\ref{tab:climate_runtime} and Figure~\ref{fig:climate_annual_tm} summarize the climate-emulation results, with experimental details and additional results in Appendix~\ref{app:climate_appendix}. From these results, we draw three main observations: \textbf{(a) CasEm substantially improves long-horizon prediction accuracy.} Water-subsystem guidance reduces full-state TM errors by $66.6\%$ for ACE and $46.3\%$ for S-DY, and total-water TM errors by $82.1\%$ and $69.4\%$, respectively. CasEm also improves ensemble predictions of 10-year mean fields, reducing S-DY's full-state fair CRPS by $62.2\%$. Water-only correction further improves S-DY's non-water predictions overall (Appendix~\ref{app:climate_group_spatial}), suggesting that gains propagate through cross-variable coupling. 
\begin{figure}[t]
\centering
\includegraphics[width=0.95\linewidth]{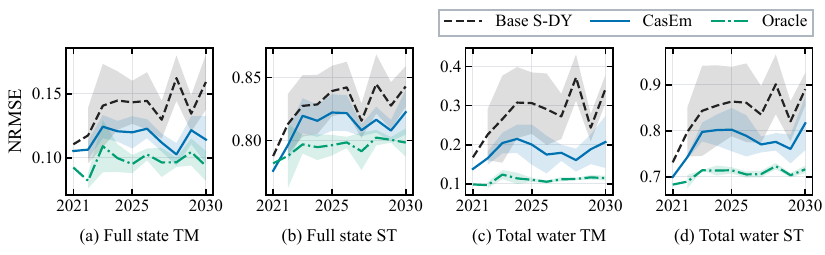}
\caption{Annual TM and ST NRMSE for 2021--2024 starts on the evaluation trajectory. Lines: means; shading: $\pm$ sample SD across starts available each year.}
\label{fig:climate_annual_tm}
\end{figure}
\textbf{(b) These gains incur little additional computational cost.} CasEm adds only $2.6\%$ and $1.4\%$ to the inference times of ACE and S-DY, respectively. Notably, ACE with CasEm achieves lower full-state ST and TM errors than uncorrected S-DY, while requiring only about one-third of its inference time. \textbf{(c) The improvements are consistent across initialization times and trajectories.} Across four initialization times on the same evaluation trajectory, Figure~\ref{fig:climate_annual_tm} shows lower mean annual TM and ST errors in every year for both the full state and total water, with generally smaller variation across starts. Both backbones improve all four metrics on each of the other 10 training trajectories, reducing mean full-state TM errors by $62.4\%$ for ACE and $43.2\%$ for S-DY (Appendix~\ref{app:climate_all11}). All 25 S-DY stochastic members likewise improve across all four rollout metrics.

\begingroup

\subsection{Regional Case Study}
\label{sec:climate_case}

Over South Asia, atmospheric moisture supply and transport play an important role in the intraseasonal variability of the Indian summer monsoon \citep{pathak2017moisture}. In the case presented here, Base S-DY produces excessively strong and frequent slow total water path (TWP) anomalies.

The regional analysis uses a 10-year rollout 
over $10$--$20^\circ$N, $65$--$95^\circ$E, a domain previously used for monsoon active--break diagnostics \citep{webster1998monsoons}, covering all June--August periods from 2021 to 2030. We evaluate slow anomalies of area-weighted TWP after removing the seasonal background and trend; processing and metric definitions are given in Appendix~\ref{app:climate_case}.

\begin{table}[!htbp]
\centering
\begin{minipage}[t]{0.73\linewidth}
\vspace{0pt}
\centering
\footnotesize
\setlength{\tabcolsep}{3pt}
\renewcommand{\arraystretch}{1.0}
\begin{tabular}{@{}lrrr@{}}
\toprule
Metric & Truth & S-DY & CasEm \\
\midrule
Slow-anomaly SD (kg m$^{-2}$) & 1.4881 & 2.7434 & 1.6826 (84.5\% $\uparrow$) \\
P95 exceedance frequency (\%) & 5.0000 & 30.5435 & 6.3043 (94.9\% $\uparrow$) \\
Regional slow RMSE (kg m$^{-2}$) & -- & 3.2665 & 1.9162 (41.3\% $\uparrow$) \\
Native-grid daily RMSE (kg m$^{-2}$) & -- & 8.7917 & 7.1828 (18.3\% $\uparrow$) \\
\midrule
Transport RMSE (kg m$^{-1}$ s$^{-1}$) & -- & 229.9891 & 194.3430 (15.5\% $\uparrow$) \\
Convergence RMSE (kg m$^{-2}$ day$^{-1}$) & -- & 6.4330 & 5.5570 (13.6\% $\uparrow$) \\
Wind vector RMSE (m s$^{-1}$) & -- & 4.6697 & 4.4292 (5.2\% $\uparrow$) \\
\bottomrule
\end{tabular}
\end{minipage}\hfill%
\begin{minipage}[t]{0.25\linewidth}
\vspace{0pt}
\small

\caption{South Asian summer diagnostics. Parentheses: relative error or bias reductions versus Base S-DY ($\uparrow$). Final three rows: unfiltered daily regional means.}
\label{tab:case_south_asia_gain}
\end{minipage}
\end{table}

\begingroup

\begin{figure}[!htbp]
\centering
\includegraphics[width=\linewidth,trim=0 12bp 0 0,clip]{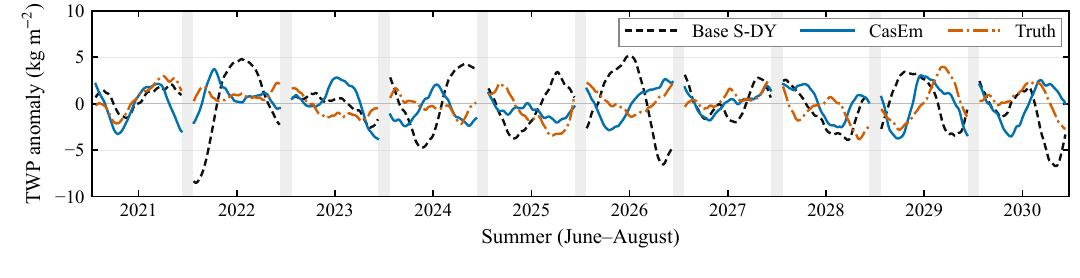}
\caption{Area-weighted slow TWP anomalies over South Asia during June--August. 
}
\label{fig:case_south_asia_jja}
\end{figure}
\endgroup
\textbf{Results.} Figure~\ref{fig:case_south_asia_jja} shows that Base S-DY produces frequent, large-amplitude excursions in regional slow TWP anomalies. CasEm substantially reduces these excessive fluctuations, bringing the anomaly trajectory closer to the reference.
Table~\ref{tab:case_south_asia_gain} quantifies this improvement: amplitude error reduced by $84.5\%$ and P95 exceedance-frequency bias by $94.9\%$. Regional slow-anomaly root mean square error (RMSE) also decreases by $41.3\%$, showing that the improvement extends beyond variability statistics to closer agreement with the reference trajectory. Daily regional total-water transport and convergence errors also decrease, showing that the improvements extend beyond water-content statistics to transport-related diagnostics. CasEm also brings regional-mean wind time series closer to the reference, showing benefits for related variables that are not directly corrected.

\FloatBarrier
\endgroup

\subsection{Ablation Studies}
\label{sec:climate_ablation}

Here, we assess whether effective subsystem guidance requires evolving dynamics and spatial structure, rather than temporal or spatial means.

\textbf{Experimental setup.} All ablations use 10-year rollouts on the same evaluation trajectory with a frozen S-DY backbone. \textbf{(a) Temporal dynamics.} Monthly, quarterly, and annual target averages test the temporal resolution required for correction. A one-quarter cyclic shift of quarterly mean targets within each year tests seasonal timing. Repeating first-year targets tests whether a single year's seasonal cycle suffices or whether modeling year-to-year evolution is necessary. \textbf{(b) Spatial structure.} We test horizontal resolution by retraining the subsystem on different target grids. We extend this resolution study to the vertical direction by aggregating the 8 layer targets into 4 adjacent-layer groups, retraining the subsystem, and applying a shared correction within each group. Finally, we test whether effective guidance depends only on the net correction or also on how it is distributed across layers. We retain the original 8-layer forecasts but sum the layerwise correction increments and redistribute the sum using a fixed mean vertical profile.

\begin{table}[!htbp]
\caption{TM error change (\%) from full CasEm ($12\!\times\!24$, 8 vertical groups); positive is worse.}
\label{tab:climate_ablation_main}
\centering
\footnotesize
\setlength{\tabcolsep}{1.5pt}
\renewcommand{\arraystretch}{1.0}
\newcommand{\climateablationrefband}[1]{{\color{black!18}\rule{3.5pt}{#1}}}
\newcommand{\climateablationbodyband}{{\color{black!18}\rule[-\dp\strutbox]{3.5pt}{\dimexpr\ht\strutbox+\dp\strutbox\relax}}}
\begin{tabular*}{\linewidth}{@{}l@{\extracolsep{\fill}}rrrrrrrcrrr@{}}
\toprule
& \multicolumn{5}{c}{Temporal} & \multicolumn{4}{c}{Horizontal space} & \multicolumn{2}{c}{Vertical space} \\
\cmidrule(lr){2-6}\cmidrule(lr){7-10}\cmidrule(lr){11-12}
Metric & \shortstack{Monthly\\mean} & \shortstack{Quarterly\\mean} & \shortstack{Annual\\mean} & \shortstack{Seasonal\\phase shift} & \shortstack{Year-1\\repeat} & $1\!\times\!1$ & $6\!\times\!12$ & \makebox[4pt][c]{\shortstack{\raisebox{3pt}[0pt][0pt]{ref.}\\[-1pt]\tikzmarknode[inner sep=0pt,outer sep=0pt]{climate-ablation-ref-head}{\phantom{\climateablationrefband{1.15ex}}}}} & $18\!\times\!36$ & \shortstack{Half\\resolution} & \shortstack{Fixed\\profile} \\
\midrule
Full-state TM & $-4.5$ & $+4.0$ & $+15.6$ & $+56.1$ & $+100.2$ & $+56.9$ & $+15.5$ & \tikzmarknode[inner sep=0pt,outer sep=0pt]{climate-ablation-ref-one}{\phantom{\climateablationbodyband}} & $-0.4$ & $+508.3$ & $+3288.6$ \\
Total-water TM & $-2.0$ & $-2.7$ & $+6.7$ & $+21.7$ & $+243.1$ & $+105.0$ & $+21.1$ & \tikzmarknode[inner sep=0pt,outer sep=0pt]{climate-ablation-ref-two}{\phantom{\climateablationbodyband}} & $+0.3$ & $+864.9$ & $+3854.6$ \\
\bottomrule
\end{tabular*}
\begin{tikzpicture}[remember picture,overlay]
  \fill[black!18] (climate-ablation-ref-head.north west) rectangle (climate-ablation-ref-two.south east);
  \draw[black,line width=\lightrulewidth]
    ([yshift={\belowrulesep+0.5\lightrulewidth}]climate-ablation-ref-one.north west) --
    ([yshift={\belowrulesep+0.5\lightrulewidth}]climate-ablation-ref-one.north east);
\end{tikzpicture}

\end{table}

\textbf{Results and analysis.} Table~\ref{tab:climate_ablation_main} summarizes relative changes in TM error; absolute ST and TM errors are reported in Appendix~\ref{app:climate_ablation}. Monthly and quarterly averaging largely preserve performance, suggesting that effective subsystem guidance relies primarily on coarse temporal structure rather than fine-scale fluctuations. However, temporal averaging cannot be made arbitrarily coarse: annual averaging increases TM errors, indicating that at least some subannual variability must be retained.
Shifting quarterly mean targets increases errors relative to the aligned quarterly means, highlighting the importance of seasonal timing. Repeating first-year targets also markedly increases errors, indicating that a fixed seasonal cycle alone does not retain the gains from year-specific evolution.
Coarser target grids degrade prediction performance, whereas refining the grid from $12\times24$ to $18\times36$ yields little additional improvement. This highlights the need to balance prediction accuracy and subsystem size, consistent with the analysis in Section~\ref{sec:subsystem_choice}.
Both halving the subsystem's vertical resolution and using a fixed vertical correction profile substantially degrade prediction performance. With the original subsystem forecasts retained, the severe degradation under fixed-profile redistribution shows that the net correction alone is far from sufficient: its allocation across layers is critical. Moisture content and the relative roles of surface exchange, turbulent mixing, and convective transport vary markedly with height \citep{beljaars2003hydrological}. Coarser vertical targets may obscure these layer-specific differences, while fixed-profile redistribution may misallocate corrections as the vertical moisture distribution evolves.

%% file: sections/limitations_conclusion.tex
\section{Conclusion and Future Work}
\label{sec:conclusion}

CasEm combines independently evolving low-dimensional subsystems with full-state backbones to improve long-horizon emulation. Our analysis provides mechanistic insight into CasEm's potential gains and offers guidance for subsystem selection. Experiments across ODE/PDE benchmarks and large-scale climate emulation demonstrate gains in long-horizon accuracy and rollout stability.
Future work will explore additional climate subsystems capturing cross-variable dynamics, organized into parallel branches or multilevel cascades.
Uncertainty-aware coordination could adapt correction strengths and balance guidance across subsystems.

%% file: sections/theory_appendix.tex
\section{Cascade Interconnections}
\label{app:cascade_background}

\subsection{Continuous-time definition}

Cascade interconnections can comprise multiple stages, with upstream states driving downstream dynamics without reverse feedback. The two-stage case, which underlies the CasEm architecture considered here, takes the following continuous-time form with upstream state $x^{(1)}$, downstream state $x^{(2)}$, and optional exogenous input $u(t)$:
\begin{equation*}
\left\{
\begin{aligned}
    \dot x^{(1)}(t)&=f^{(1)}\!\left(x^{(1)}(t),u(t)\right),\\
    \dot x^{(2)}(t)&=f^{(2)}\!\left(x^{(2)}(t),x^{(1)}(t),u(t)\right).
\end{aligned}
\right.
\end{equation*}
Its defining property, illustrated in Figure~\ref{fig:cascade_continuous}, is the triangular dependence: the downstream state $x^{(2)}$ may depend on the upstream state $x^{(1)}$, whereas the upstream dynamics do not depend on $x^{(2)}$. This is stronger than merely evaluating two modules in sequence. It is a structural restriction on their recurrent dynamics, and it prevents downstream disturbances or modeling errors from feeding back into the upstream state through the interconnection.

\input{figures/cascade_continuous}

\subsection{Upstream-first discretization}

The one-way dependence permits a partitioned time discretization that advances the upstream block before the downstream block. On a grid with step size $h$, let $\Phi_h^{(1)}$ and $\Phi_h^{(2)}$ denote the corresponding one-step maps for the two blocks. An upstream-first realization is
\begin{equation}
\left\{
\begin{aligned}
    x^{(1)}_{t+1}&=\Phi_h^{(1)}(x^{(1)}_t,u_t),\\
    x^{(2)}_{t+1}&=\Phi_h^{(2)}(x^{(2)}_t,x^{(1)}_{t+1},u_t).
\end{aligned}
\right.
\label{eq:sequential_cascade}
\end{equation}
Here $x^{(1)}_{t+1}$ is not unavailable future information: it has already been produced by the upstream update within the current time step. A simultaneous explicit discretization may instead condition the downstream update on $x^{(1)}_t$. Both conventions preserve the cascade direction; Eq.~(\ref{eq:sequential_cascade}) is the ordered convention used when the downstream state at $t+1$ should be coordinated with the upstream state at the same time. CasEm follows this convention in Eq.~(\ref{eq:casem_rollout}): $g_\psi$ advances the upstream subsystem, while $f_\theta$ and $\mathcal C_\phi$ jointly update the downstream full state.

\subsection{Applications and structural advantages}

Cascade architectures are a well-established way to organize complex systems, with classical applications in control and signal processing \citep{astrom2008feedback}. For example, high-order filters can be implemented as successive low-order stages, simplifying implementation and improving numerical robustness \citep{lyons2004understanding}.

The appeal of this structure lies in modularity: different stages can perform specialized roles, use different models, and interact through explicit interfaces. This organization supports component-wise design and analysis, allowing the behavior of a complex system to be understood through its constituent subsystems and their interconnections \citep{khalil2002nonlinear}. CasEm applies this modular structure by placing the independently evolving low-dimensional subsystem upstream of the full-state emulator.

\section{Theoretical Results and Proofs}
\label{app:formal_theory}
\label{app:formal_setting}
\label{app:formal_theorem}

This appendix gives the mathematical details and proof of Theorem~\ref{thm:main}, an interpretation of its error bound, and an analysis of subsystem rollout error under approximate closure.

\subsection{Theoretical setup}
\label{app:theory_setup}
\label{app:error_definitions}
\label{app:hard_writeback}
\label{app:key_ratios}
\label{app:complement_sensitivity}

The notation follows Section~\ref{sec:method}. Here we give explicit mathematical forms for the accumulated errors, normalization, correction rule, and sensitivity coefficient used in Theorem~\ref{thm:main}.

\paragraph{Notation and normalization.}
We evaluate $M$ reference trajectories over a horizon $T$. For each reference trajectory, the base, CasEm, and Oracle rollouts start from its initial state and use its external forcing. States are expressed in fixed normalized coordinates so that the selected full-state metric is Euclidean; any fixed variable or spatial weights are absorbed into these coordinates.

The orthonormal-row condition ($HH^\top=I_n$) in
Theorem~\ref{thm:main} can be satisfied by transforming
any full-row-rank readout $\bar H$ as
\begin{equation*}
    H=(\bar H\bar H^\top)^{-1/2}\bar H.
\end{equation*}
The same transformation is applied to true and predicted
subsystem targets. The matrices
\begin{equation*}
    P_H=H^\top H,
    \qquad
    Q_H=I_N-P_H
\end{equation*}
are the orthogonal projections onto the aggregate directions and their complementary subspace, respectively. These operators act pointwise on trajectory collections.

\paragraph{Accumulated errors and diagnostic ratios.}
Let $X^\star=\{x_{i,t}\}_{i,t}$ collect the reference trajectories, and let $X^{\mathrm{base}}$, $X^{\mathrm{cas}}$, and $X^{\mathrm{cas,orc}}$ collect the base, CasEm, and Oracle rollouts described in Section~\ref{sec:theory_main}. Write $Z^\star=HX^\star$ for the true subsystem targets and $\widehat Z$ for the independent subsystem forecasts. For any trajectory collection $A=\{a_{i,t}\}_{i,t}$, define
\begin{equation*}
    \|A\|_T^2
    :=
    \sum_{i=1}^{M}\sum_{t=1}^{T}\|a_{i,t}\|_2^2.
\end{equation*}
The accumulated errors introduced in Section~\ref{sec:theory_main} then have the explicit forms
\begin{equation*}
\begin{aligned}
    \mathcal E_x^{\mathrm{base}}
    &=\|X^{\mathrm{base}}-X^\star\|_T^2,\\
    \mathcal E_x^{\mathrm{cas}}
    &=\|X^{\mathrm{cas}}-X^\star\|_T^2,\\
    \mathcal E_x^{\mathrm{cas,orc}}
    &=\|X^{\mathrm{cas,orc}}-X^\star\|_T^2,\\
    \mathcal E_z^{\mathrm{base}}
    &=\|HX^{\mathrm{base}}-Z^\star\|_T^2,\\
    \mathcal E_z^{\mathrm{sub}}
    &=\|\widehat Z-Z^\star\|_T^2.
\end{aligned}
\end{equation*}
All errors use equal trajectory and time weights. Dividing them by the common factor $MT$ leaves the diagnostic ratios unchanged. Errors are accumulated before ratios are formed, rather than averaging per-step ratios.

For $\mathcal E_x^{\mathrm{base}}>0$ and $\mathcal E_z^{\mathrm{base}}>0$, the ratios $\rho,q,\gamma$ defined in Section~\ref{sec:theory_main} satisfy
\begin{equation*}
    \frac{\mathcal E_z^{\mathrm{sub}}}{\mathcal E_x^{\mathrm{base}}}
    =\rho q,
    \qquad
    \frac{\mathcal E_x^{\mathrm{cas,orc}}}{\mathcal E_x^{\mathrm{base}}}
    =1-\rho\gamma.
\end{equation*}
In particular, nonnegativity of the Oracle error implies $\rho\gamma\leq1$. If $\mathcal E_z^{\mathrm{base}}=0<\mathcal E_x^{\mathrm{base}}$, $q$ and $\gamma$ are undefined; the corresponding identifiable quantities are
\begin{equation*}
    \frac{\mathcal E_z^{\mathrm{sub}}}{\mathcal E_x^{\mathrm{base}}},
    \qquad
    1-\frac{\mathcal E_x^{\mathrm{cas,orc}}}{\mathcal E_x^{\mathrm{base}}}.
\end{equation*}
If $\mathcal E_x^{\mathrm{base}}=0$, only the unnormalized error bound below applies.

\paragraph{Rollout sensitivity.}
The constant $B$ in Theorem~\ref{thm:main} bounds the amplification of subsystem-target errors into complementary directions during rollout. For a fixed backbone and correction map $\mathcal C_H$, specified below, supplying a target collection $Z=\{\zeta_{i,t}\}_{i,t}$ produces the recurrence
\begin{equation*}
    \widehat x_{i,0}^{H,Z}=x_{i,0},
    \qquad
    \widehat x_{i,t}^{H,Z}
    =
    \mathcal C_H\!\left(
        f_\theta(\widehat x_{i,t-1}^{H,Z},u_{i,t-1}),
        \zeta_{i,t}
    \right).
\end{equation*}
Denote the resulting map from targets to full-state trajectories by
\begin{equation*}
    \Phi_{H,T}(Z)
    =
    X^{H,Z}
    =
    \{\widehat x_{i,t}^{H,Z}\}_{i,t}.
\end{equation*}
The actual and Oracle rollouts therefore satisfy
\begin{equation*}
    X^{\mathrm{cas}}=\Phi_{H,T}(\widehat Z),
    \qquad
    X^{\mathrm{cas,orc}}=\Phi_{H,T}(Z^\star).
\end{equation*}

Choose $r>0$ such that $\|\widehat Z-Z^\star\|_T\leq r$, and suppose the corrected recurrence is well defined for all target collections in this neighborhood. Define
\begin{equation}
B_H(T;r)
:=
\sup_{0<\|Z-Z^\star\|_T\leq r}
\frac{
    \left\|Q_H\!\left[
        \Phi_{H,T}(Z)-\Phi_{H,T}(Z^\star)
    \right]\right\|_T
}{
    \|Z-Z^\star\|_T
}.
\label{eq:B_definition}
\end{equation}
Thus $B_H(T;r)\sqrt{\mathcal E_z^{\mathrm{sub}}}$ bounds the difference between the CasEm and Oracle rollouts in complementary directions. Theorem~\ref{thm:main} assumes $B=B_H(T;r)<\infty$; its magnitude depends on the corrected dynamics.

\paragraph{Correction rule.}
Theorem~\ref{thm:main} establishes the bound for linear readouts with exact minimum-change correction, a specific instance of CasEm's general coordination scheme. Under $HH^\top=I_n$, its explicit form is
\begin{equation*}
\begin{aligned}
    \mathcal C_H(\widetilde x,\zeta)
    &=
    \operatorname*{arg\,min}_{y:\,Hy=\zeta}
    \|y-\widetilde x\|_2^2\\
    &=
    \widetilde x+H^\top(\zeta-H\widetilde x)\\
    &=
    Q_H\widetilde x+H^\top\zeta.
\end{aligned}
\end{equation*}
Consequently,
\begin{equation*}
    H\mathcal C_H(\widetilde x,\zeta)=\zeta,
    \qquad
    Q_H\mathcal C_H(\widetilde x,\zeta)=Q_H\widetilde x.
\end{equation*}
For a target $\bar\zeta$ in the original aggregate coordinates, the same correction can be written as
\begin{equation*}
    \widetilde x+\bar H^\dagger(\bar\zeta-\bar H\widetilde x),
    \qquad
    \zeta=(\bar H\bar H^\top)^{-1/2}\bar\zeta.
\end{equation*}
Whitening therefore changes the aggregate coordinates, not the full-state correction.

Although this correction preserves the complementary component of each proposal, the corrected state enters subsequent backbone predictions. Subsystem-target errors can therefore affect complementary directions at later steps, which is why the rollout sensitivity $B$ remains necessary.

\subsection{Proof of Theorem 1}

For completeness, we restate the precise version of Theorem~\ref{thm:main}. If $\|\widehat Z-Z^\star\|_T\leq r$ and $B_H(T;r)<\infty$, then
\begin{equation}
\boxed{
\mathcal E_x^{\mathrm{cas}}
\leq
\left[
    \sqrt{\mathcal E_x^{\mathrm{cas,orc}}}
    +B_H(T;r)\sqrt{\mathcal E_z^{\mathrm{sub}}}
\right]^2
+\mathcal E_z^{\mathrm{sub}}.
}
\label{eq:unnormalized_bound}
\end{equation}
For positive denominators, equivalently,
\begin{equation}
\boxed{
\frac{\mathcal E_x^{\mathrm{cas}}}{\mathcal E_x^{\mathrm{base}}}
\leq
\left[
    \sqrt{1-\rho\gamma}
    +B_H(T;r)\sqrt{\rho q}
\right]^2
+\rho q.
}
\label{eq:normalized_bound_appendix}
\end{equation}

\begin{proof}[\normalfont\bfseries Proof of Theorem~\ref{thm:main}]

Let the target-sequence error and the learned--Oracle trajectory difference be
\begin{equation*}
    \Delta Z=\widehat Z-Z^\star,
    \qquad
    \Delta X=X^{\mathrm{cas}}-X^{\mathrm{cas,orc}}.
\end{equation*}
Let the Oracle Cascade trajectory error and the learned-target-induced complement difference be
\begin{equation*}
    e^{\mathrm{cas,orc}}=X^{\mathrm{cas,orc}}-X^\star,
    \qquad
    D=Q_H\Delta X.
\end{equation*}
Because the Oracle Cascade writes back the true target at every step,
\begin{equation*}
    HX^{\mathrm{cas,orc}}=Z^\star=HX^\star,
\end{equation*}
which gives
\begin{equation*}
    He^{\mathrm{cas,orc}}=0,
    \qquad
    e^{\mathrm{cas,orc}}=Q_H e^{\mathrm{cas,orc}}.
\end{equation*}
The learned Cascade writes back $\widehat Z$, and therefore
\begin{equation*}
    H\Delta X
    =HX^{\mathrm{cas}}-HX^{\mathrm{cas,orc}}
    =\widehat Z-Z^\star
    =\Delta Z.
\end{equation*}
Using the orthogonal $P_H/Q_H$ decomposition,
\begin{align*}
    \Delta X
    &=Q_H\Delta X+P_H\Delta X\\
    &=D+H^\top H\Delta X\\
    &=D+H^\top\Delta Z.
\end{align*}
Thus the learned Cascade error is
\begin{align*}
    X^{\mathrm{cas}}-X^\star
    &=\left(X^{\mathrm{cas,orc}}-X^\star\right)
      +\left(X^{\mathrm{cas}}-X^{\mathrm{cas,orc}}\right)\\
    &=e^{\mathrm{cas,orc}}+D+H^\top\Delta Z.
\end{align*}
Both $e^{\mathrm{cas,orc}}$ and $D$ lie in the $Q_H$ trajectory subspace, while $H^\top\Delta Z$ lies in the orthogonal $P_H$ trajectory subspace. Hence
\begin{equation*}
    \left\langle e^{\mathrm{cas,orc}}+D,H^\top\Delta Z\right\rangle_T=0.
\end{equation*}
The row-orthogonality assumption $HH^\top=I$ also makes $H^\top$ norm-preserving on target space:
\begin{equation*}
    \|H^\top\Delta Z\|_T^2
    =\langle\Delta Z,HH^\top\Delta Z\rangle_T
    =\|\Delta Z\|_T^2.
\end{equation*}
Pythagoras therefore yields
\begin{equation}
    \mathcal E_x^{\mathrm{cas}}
    =\|e^{\mathrm{cas,orc}}+D\|_T^2
     +\|\Delta Z\|_T^2.
    \label{eq:pythagorean_decomposition}
\end{equation}
The triangle inequality gives
\begin{equation*}
    \|e^{\mathrm{cas,orc}}+D\|_T
    \leq
    \|e^{\mathrm{cas,orc}}\|_T+\|D\|_T.
\end{equation*}
By definition,
\begin{equation*}
    \|e^{\mathrm{cas,orc}}\|_T
    =\sqrt{\mathcal E_x^{\mathrm{cas,orc}}}.
\end{equation*}
If $\widehat Z=Z^\star$, then $D=0$. Otherwise, since $\|\widehat Z-Z^\star\|_T\leq r$, substituting $Z=\widehat Z$ into Eq.~(\ref{eq:B_definition}) gives the following bound, which holds in either case:
\begin{align*}
    \|D\|_T
    &=\left\|Q_H\left[
        \Phi_{H,T}(\widehat Z)-\Phi_{H,T}(Z^\star)
      \right]\right\|_T\\
    &\leq B_H(T;r)\|\widehat Z-Z^\star\|_T\\
    &=B_H(T;r)\sqrt{\mathcal E_z^{\mathrm{sub}}}.
\end{align*}
Consequently,
\begin{equation*}
    \|e^{\mathrm{cas,orc}}+D\|_T
    \leq
    \sqrt{\mathcal E_x^{\mathrm{cas,orc}}}
    +B_H(T;r)\sqrt{\mathcal E_z^{\mathrm{sub}}}.
\end{equation*}
Squaring this inequality and substituting it into Eq.~(\ref{eq:pythagorean_decomposition}) proves Eq.~(\ref{eq:unnormalized_bound}). For positive denominators, dividing by $\mathcal E_x^{\mathrm{base}}$ and using
\begin{equation*}
    \frac{\mathcal E_x^{\mathrm{cas,orc}}}{\mathcal E_x^{\mathrm{base}}}
    =1-\rho\gamma,
    \qquad
    \frac{\mathcal E_z^{\mathrm{sub}}}{\mathcal E_x^{\mathrm{base}}}
    =\rho q
\end{equation*}
proves Eq.~(\ref{eq:normalized_bound_appendix}).
\end{proof}

\subsection{Interpretation of Theorem 1}

The bound compares the benefit of exact subsystem guidance with the cost of using imperfect subsystem predictions. Expanding it gives
\begin{equation*}
\frac{\mathcal E_x^{\mathrm{cas}}}{\mathcal E_x^{\mathrm{base}}}
\leq
\underbrace{1-\rho\gamma}_{\text{ideal corrected error}}
+\underbrace{\rho q}_{\text{direct subsystem error}}
+\underbrace{2B\sqrt{\rho q(1-\rho\gamma)}+B^2\rho q}_{\text{propagation cost}}.
\end{equation*}
Thus, a sufficient condition for improvement is that the
ideal correction gain $\rho\gamma$ exceeds the direct
subsystem error and its propagation cost.

\paragraph{Simplified conditions.}
Without target-error propagation into complementary directions ($B=0$), this condition reduces to $q<\gamma$ for $\rho>0$: subsystem prediction error must be smaller than the improvement delivered by exact-target correction. If exact-target correction eliminates that aggregate error without changing the base model's complementary rollout, then $\gamma=1$, and the condition further simplifies to $q<1$, meaning that the subsystem prediction error is smaller than the base model's error in the selected aggregate directions.

\paragraph{Effect of error propagation.}
With nonzero propagation, greater subsystem accuracy is needed to compensate for downstream sensitivity. Importantly, a small propagation contribution does not require a small $B$: sufficiently accurate subsystem targets can limit the propagated error even when the downstream response is appreciable. A separate Heat/FNO diagnostic illustrates this mechanism. The observed response is of order one, but the subsystem-target error is sufficiently small that the corresponding propagation contribution is negligible relative to the gain margin. Direct diagnostics also show that exact-target correction
nearly eliminates aggregate-direction error while leaving
the complementary component of the rollout close to that
of the base model. The full-state error reduction is therefore
approximately equal to the base model's aggregate-direction
error, yielding $\gamma\approx1$. These observations support the mechanism interpretation but do not establish a uniform upper bound on $B$.

\subsection{Approximate closure and subsystem rollout error}
\label{app:closure_q}

Recall from Section~\ref{sec:method} that the true aggregate variables satisfy an approximately closed relation with a state-dependent residual. We now make the closure radius precise and show how closure, dynamics-learning error, and autonomous amplification jointly control subsystem rollout. Over the relevant reachable region $\Omega_t$, define the fiber associated with an aggregate state $z$ by
\begin{equation*}
    \mathfrak F_t(z)=\{x\in\Omega_t:Hx=z\},
\end{equation*}
and the state-dependent closure radius by
\begin{equation*}
    \varepsilon_{\mathrm{cl},t}(z,u)
    :=
    \inf_v\sup_{x\in\mathfrak F_t(z)}
    \left\|H\mathcal F(x,u)-v\right\|_2.
\end{equation*}
For the states and inputs considered, assume that each fiber is nonempty and its projected image $\{H\mathcal F(x,u):x\in\mathfrak F_t(z)\}$ is bounded. The infimum is then attained in $\mathbb R^n$; let $G_t^\star(z,u)$ be a minimizing center. Along a true trajectory,
\begin{equation*}
    z_{t+1}=G_t^\star(z_t,u_t)+r_t(x_t,u_t),
    \qquad
    \|r_t(x_t,u_t)\|_2
    \leq
    \varepsilon_{\mathrm{cl},t}(z_t,u_t).
\end{equation*}
To expose the source of this residual, suppose the projected dynamics within a common fiber obey
\begin{equation*}
    \left\|H\mathcal F(x,u)-H\mathcal F(x',u)\right\|_2
    \leq
    \kappa_{\mathrm{cl},t}(z,u)
    \left\|Q_H(x-x')\right\|_2,
    \qquad x,x'\in\mathfrak F_t(z).
\end{equation*}
Define the unresolved fiber spread
\begin{equation*}
    d_{Q,t}(z)
    :=
    \sup_{x,x'\in\mathfrak F_t(z)}
    \left\|Q_H(x-x')\right\|_2.
\end{equation*}
Then
\begin{equation*}
    \varepsilon_{\mathrm{cl},t}(z,u)
    \leq
    \kappa_{\mathrm{cl},t}(z,u)d_{Q,t}(z).
\end{equation*}
Thus approximate closure is governed jointly by unresolved-state spread and its coupling into the aggregate dynamics, rather than by a state-independent absolute residual.

\begin{theorem}[Subsystem error under approximate closure]
\label{thm:approx_closure}
Suppose the learned subsystem satisfies
\begin{equation*}
    \widehat z_{t+1}=g_\psi(\widehat z_t,u_t),
\end{equation*}
and define the pointwise dynamics-learning error
\begin{equation*}
    \delta_{\mathrm{dyn},t}(z,u)
    :=
    \left\|g_\psi(z,u)-G_t^\star(z,u)\right\|_2.
\end{equation*}
If, for each $t$, $g_\psi(\cdot,u_t)$ is $L_{z,t}$-Lipschitz on a domain containing both $z_t$ and $\widehat z_t$, then the subsystem rollout error $\eta_t=\widehat z_t-z_t$ obeys
\begin{equation}
\boxed{
\begin{aligned}
\|\eta_t\|_2
&\leq
\Gamma_z(t,0)\|\eta_0\|_2\\
&\quad+\sum_{k<t}\Gamma_z(t,k+1)
\left[
\delta_{\mathrm{dyn},k}(z_k,u_k)
+\kappa_{\mathrm{cl},k}(z_k,u_k)d_{Q,k}(z_k)
\right],
\end{aligned}
}
\label{eq:approx_closure_bound}
\end{equation}
where
\begin{equation*}
    \Gamma_z(t,s)=\prod_{\ell=s}^{t-1}L_{z,\ell}.
\end{equation*}
\end{theorem}

\begin{proof}
Adding and subtracting $g_\psi(z_t,u_t)$ and applying the Lipschitz condition, the pointwise dynamics-learning error, and the fiber closure-gain bound gives
\begin{equation*}
    \|\eta_{t+1}\|_2
    \leq
    L_{z,t}\|\eta_t\|_2
    +\delta_{\mathrm{dyn},t}(z_t,u_t)
    +\kappa_{\mathrm{cl},t}(z_t,u_t)d_{Q,t}(z_t).
\end{equation*}
Iterating this recursion yields Eq.~(\ref{eq:approx_closure_bound}).
\end{proof}

Theorem~\ref{thm:approx_closure} is not a premise of Theorem~\ref{thm:main}. Theorem~\ref{thm:main} may use the empirically measured relative subsystem prediction error $q$ directly; the present result only provides an optional analytical bound derived from closure, dynamics-learning error, and autonomous stability. Taking the supremum of $\kappa_{\mathrm{cl},t}(z,u)d_{Q,t}(z)$ over the relevant states and inputs gives a uniform bound on the closure-residual term at each time $t$. Exact closure removes the closure forcing, but not dynamics-learning error or autonomous amplification through the factors $\Gamma_z$.

\clearpage
\section{Supplementary Methodological Details}
\label{app:method_comparison}

\subsection{Subsystem selection}
\label{app:subsystem_screening}

This section expands the diagnostics in
Section~\ref{sec:theory_main} into practical guidance for
selecting subsystem variables, modeling their dynamics,
and designing corrections, as summarized in
Figure~\ref{fig:subsystem_screening}.

\paragraph{(a) Error coverage $\rho$: identifying promising correction directions.}
Estimate $\rho$ from Base rollouts to identify aggregate
directions that account for an appreciable share of
accumulated prediction error. If coverage is limited,
reconsider the selected variables, regional partition,
or spectral bands.

\paragraph{(b) Modelability $q$: obtaining reliable forecasts with compact models.}
Prefer low-complexity models that adequately capture the
subsystem dynamics, such as a compact neural network or,
when compatible with those dynamics, a Schur-stable linear
model. Model capacity and stability constraints should be
assessed through autonomous multistep forecasts rather than
one-step fitting accuracy alone. If $q$ remains large,
consider both improving the predictor and revising the
subsystem variables; increasing network size is not the
only route to better modelability.

\paragraph{(c) Correctability $\gamma$: designing effective full-state corrections.}
For a fixed aggregate readout and backbone, compare candidate
correction maps through Oracle rollouts with true subsystem
targets. Options include uniform or region-dependent gains,
layer-dependent gains, and different allocations of correction
increments. Neural networks can learn more complex
state-dependent correction relationships. When multiple
correction patterns are compatible with the same subsystem
target, conditional generative models can represent their
distribution, including possible multimodal structure.
Our climate-emulation experiments provide an example through
the diffusion-based energy correction, which generates spatial
temperature adjustments conditioned on the energy mismatch,
external forcing, and relevant full-state fields.
Evaluate correction maps by their effects on full-state
rollout error rather than only their immediate agreement
with the aggregate targets. A small $\gamma$ may therefore
motivate revising the correction map rather than immediately
discarding the candidate subsystem.

Candidate composition and dimension should therefore be
chosen by balancing error coverage and modelability,
rather than maximizing coverage alone.
All diagnostic comparisons use the same development
trajectories, horizon, and normalized metric. These diagnostics
provide qualitative guidance rather than universal acceptance
thresholds. Shortlisted designs are then evaluated through
actual CasEm rollouts with predicted subsystem targets to
assess their realized full-state gains.

\begin{figure}[!ht]
    \centering
    \includegraphics[width=\textwidth]{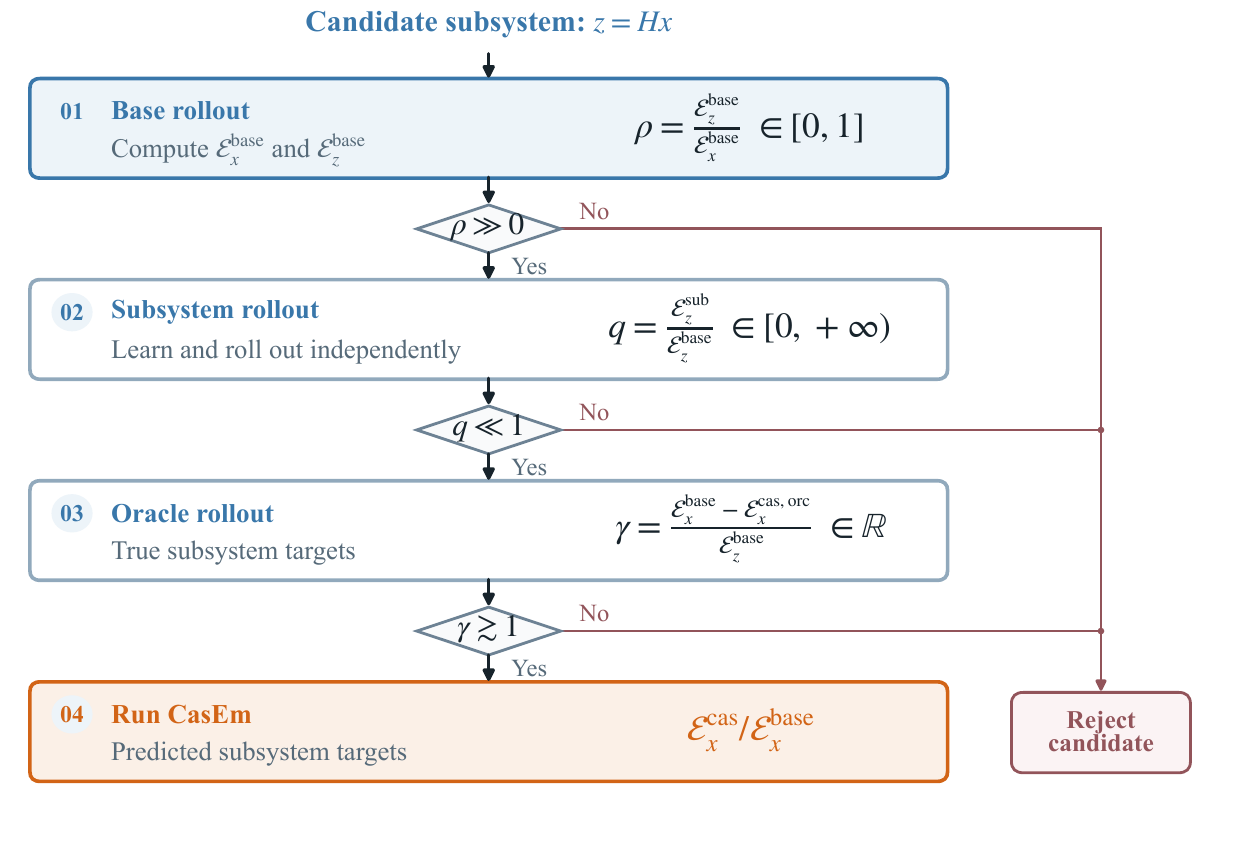}
    \caption{Subsystem screening and CasEm validation.
    The qualitative criterion $\gamma\gtrsim1$ allows values
    slightly below one but requires positive ideal correction
    gain ($\gamma>0$).}
    \label{fig:subsystem_screening}
\end{figure}
\subsection{Comparison of prediction structures}
\label{app:prediction_structures}

This section supplements the comparison of coarse-to-fine and non-autoregressive methods in Section~\ref{sec:related_work}.
Figure~\ref{fig:temporal_state_dependencies} compares state dependencies across multiple physical times, distinguishing temporal recurrence from spatial or spatiotemporal refinement.

\paragraph{(a) Coarse-scale rollout with downscaling.}
In panel (a), temporal recurrence operates on the coarse-grid state, while each fine-scale output is generated from the corresponding coarse state rather than advanced from the preceding fine-scale state \citep{lupinjimenez2025fcds,perkins2025hiroace}. Unlike CasEm's selected aggregate variables, this coarse state represents the complete system at reduced spatial resolution.

\paragraph{(b) Non-autoregressive sequence generation.}
In panel (b), multiple physical times are generated jointly. In Spatiotemporal Pyramid Flows (SPF) \citep{irvin2026pyramid}, successive refinement stages increase the sequence's spatial and temporal resolution rather than advance the system from one physical time to the next. These refinement stages therefore should not be interpreted as autoregressive time steps.

\paragraph{(c) Cascaded full-state prediction.}
In panel (c), the subsystem and full-state emulator each retain a temporal recurrence, with guidance flowing only from the subsystem to the full state. The corrected full state enters the next backbone step, distinguishing CasEm from both per-time-step downscaling and joint sequence generation.

Predicting multiple time steps at once does not by itself make a rollout non-autoregressive. If predicted states condition subsequent prediction windows, the rollout remains blockwise autoregressive, as in DYffusion beyond its training horizon \citep{ruhlingcachay2023dyffusion}.

\begin{figure}[!ht]
    \centering
    \includegraphics[width=\linewidth]{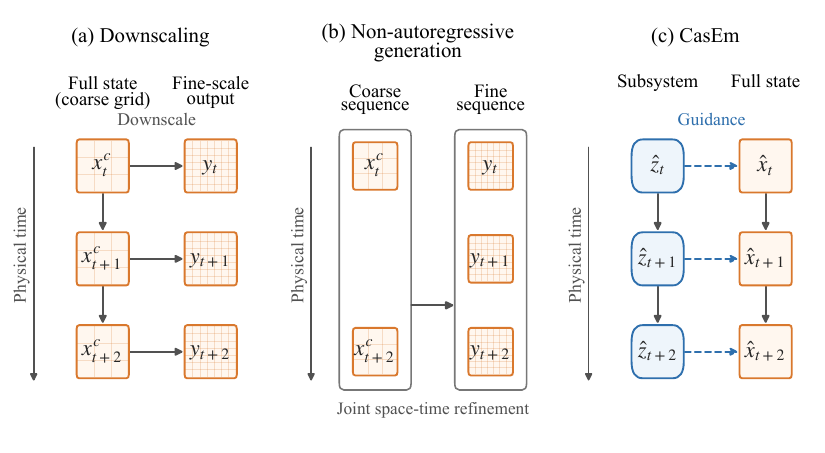}
    \caption{Temporal state dependencies: (a) coarse-grid full-state rollout with per-time-step downscaling; (b) joint sequence generation through spatiotemporal refinement; (c) CasEm's autonomous subsystem guiding recurrent full-state prediction. Physical time runs downward. Horizontal arrows denote downscaling in (a), whole-sequence refinement in (b), and subsystem guidance in (c); the refinement arrow in (b) is not a physical-time update.}
    \label{fig:temporal_state_dependencies}
\end{figure}

%% file: figures/cascade_continuous.tex
\begin{figure}[htbp]
    \centering
    \includegraphics[width=0.90\linewidth]{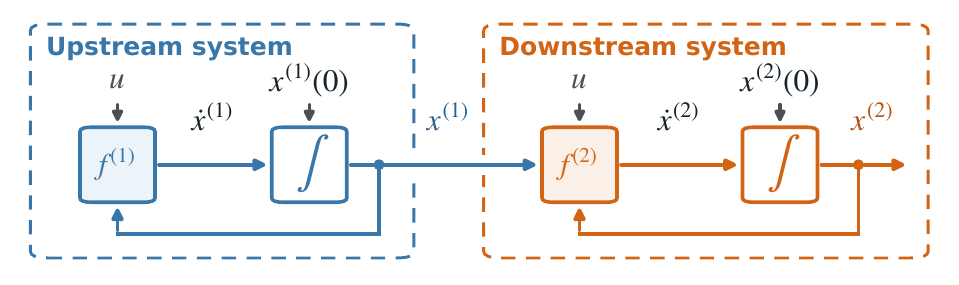}
\caption{Block diagram of a continuous-time two-stage \textbf{cascade} interconnection. The upstream state $x^{(1)}(t)$ drives the downstream dynamics $f^{(2)}$, with no feedback from the downstream system to the upstream system. Both systems receive the external input $u(t)$; integrators evolve the states from the indicated initial conditions.}
    \label{fig:cascade_continuous}
\end{figure}

%% file: sections/experimental_details_appendix.tex
\section{ODE and PDE Benchmarks}
\label{app:controlled_details}

This appendix supplements Section~\ref{sec:experiments} with benchmark definitions, task-specific designs, numerical configurations, and additional results. It also includes an additional example with nonlinear aggregation (Section~\ref{app:central_force}) and an ablation of unidirectional coupling (Section~\ref{app:unidirectional_ablation}).

\subsection{Dynamical systems and subsystem variables}
Table~\ref{tab:controlled_systems} summarizes the four benchmarks and their subsystem and backbone configurations.
\begingroup
\begin{table}[H]
\caption{Benchmark systems and specified subsystems.}
\label{tab:controlled_systems}
\centering
\footnotesize
\setlength{\tabcolsep}{2.3pt}
\resizebox{\textwidth}{!}{%
\begin{tabular}{@{}lllcll@{}}
\toprule
Case & System type & Specified subsystem & $n$ & Closure & Backbone \\
\midrule
Heat & 1D PDE & DC+Low-4 Fourier, $k=0{:}4$ & 9 & Exact & ResCNN, FNO \\
Fisher--KPP & 1D PDE & Joint low-frequency block, $k=0{:}16$ & 33 & Approximate & ResCNN, FNO \\
BVE & Spherical PDE & Real $l=1$ spherical harmonics & 3 & Exact & ResCNN, SFNO \\
Alanine NVE & ODE & Mass-weighted COM position/velocity & 6 & Exact & GNS, EGNN-velocity \\
\bottomrule
\end{tabular}}
\end{table}
\endgroup

\paragraph{Heat equation.}
The forced heat equation describes linear diffusion under an external source:
\begin{equation*}
\partial_t u(x,t)=\nu\,\partial_{xx}u(x,t)+f(x,t),
\qquad x\in[0,L),
\end{equation*}
with periodic boundary conditions. Here, $u$ is temperature, $\nu$ is thermal diffusivity, $f$ is the external source, and $L$ is the periodic domain length. Let
\begin{equation*}
a_k(t)=\frac{1}{L}\int_0^L u(x,t)e^{-2\pi ikx/L}\,dx
\end{equation*}
denote its Fourier coefficients. The subsystem retains the spatial mean and the first $K$ Fourier modes,
\begin{equation*}
z(t)=\bigl(a_0,\operatorname{Re}a_1,\operatorname{Im}a_1,\ldots,
\operatorname{Re}a_K,\operatorname{Im}a_K\bigr),
\end{equation*}
giving $2K+1$ real variables. These variables describe the mean temperature and large-scale spatial variations. Each retained coefficient satisfies
\begin{equation*}
\dot a_k=-\nu\left(\frac{2\pi k}{L}\right)^2a_k+f_k(t),
\end{equation*}
where $f_k$ is the corresponding forcing coefficient. The subsystem is therefore exactly closed given the external forcing. The main configuration uses $K=4$, yielding a 9-dimensional subsystem.

\paragraph{Fisher--KPP equation.}
The Fisher--KPP equation combines diffusion with logistic growth:
\begin{equation*}
\partial_t u(x,t)=D\,\partial_{xx}u(x,t)+r\,u(x,t)\bigl(1-u(x,t)\bigr),
\qquad x\in[0,L),
\end{equation*}
with periodic boundary conditions. Here, $u$ represents normalized population density or concentration, $D$ is the diffusion coefficient, and $r$ is the local growth rate. It describes the spreading of a population or reacting quantity through propagating fronts. The subsystem uses the same Fourier readout as Heat, retaining $k=0,\ldots,K$. These variables capture the mean concentration and the broad spatial structure of the fronts.

Unlike Heat, the nonlinear reaction term couples retained and discarded modes:
\begin{equation*}
\dot a_k=
\left[r-D\left(\frac{2\pi k}{L}\right)^2\right]a_k
-r\sum_{p\in\mathbb Z}a_p a_{k-p}.
\end{equation*}
Consequently, the retained block is not exactly closed. Its joint dynamics are learned as an approximate subsystem. The main configuration uses $K=16$, yielding 33 real variables.

\paragraph{Barotropic vorticity equation.}
The rotating-sphere barotropic vorticity equation provides a simplified model of large-scale atmospheric flow:
\begin{equation*}
\partial_t\zeta+J(\psi,\zeta+f)=0,
\qquad
\zeta=\nabla_{\mathbb S_R^2}^{2}\psi,
\qquad
f=2\Omega\sin\varphi,
\end{equation*}
where $\zeta$ is relative vorticity, $\psi$ is the streamfunction, $\varphi$ is latitude, and $J$ is the spherical Jacobian. $f$ is the Coriolis parameter, $\Omega$ is the planetary angular velocity, and $R$ is the sphere radius. Writing the vorticity in spherical harmonics, the subsystem retains the 3 real degrees of freedom at degree $l=1$:
\begin{equation*}
z(t)=\bigl(\zeta_{10}(t),\operatorname{Re}\zeta_{11}(t),
\operatorname{Im}\zeta_{11}(t)\bigr).
\end{equation*}
These coefficients represent the solid-body rotational component of the flow and are associated with its angular momentum. Their evolution is independent of the higher-degree modes: the axial component remains constant, while the other two components rotate at the planetary rotation frequency. This gives an exactly closed 3-dimensional subsystem.

\paragraph{Alanine molecular dynamics.}
The Alanine benchmark describes the constant-energy dynamics of a 22-atom alanine dipeptide:
\begin{equation*}
\dot q_i=v_i,
\qquad
m_i\dot v_i=-\nabla_{q_i}U(q)+F_i^{\mathrm{con}},
\qquad i=1,\ldots,22,
\end{equation*}
where $q_i,v_i\in\mathbb R^3$ are the position and velocity of atom $i$, $m_i$ is its mass, $U$ is the molecular potential, and $F_i^{\mathrm{con}}$ enforces the prescribed bond constraints. The complete position--velocity state has 132 components.

The subsystem consists of the center-of-mass position and velocity:
\begin{equation*}
\begin{gathered}
z=(q_{\mathrm{COM}},v_{\mathrm{COM}}),\\
q_{\mathrm{COM}}=\frac{1}{M}\sum_i m_iq_i,
\qquad
v_{\mathrm{COM}}=\frac{1}{M}\sum_i m_iv_i,
\qquad M=\sum_i m_i.
\end{gathered}
\end{equation*}
These 6 variables describe the molecule's overall translation rather than its internal conformation. Because the potential and internal constraints are translation invariant and no external force acts on the molecule,
\begin{equation*}
\dot q_{\mathrm{COM}}=v_{\mathrm{COM}},
\qquad
\dot v_{\mathrm{COM}}=0.
\end{equation*}
The translational subsystem is therefore exactly closed.

\subsection{Correction maps and subsystem-only decoders}
\label{app:correction_decoders}

\subsubsection{Full-state correction}
\label{app:full_state_correction}

Matching the subsystem forecasts constrains only part of the full state, leaving freedom in how the remaining components are treated. Our correction maps either preserve the backbone's remaining prediction or learn additional adjustments while retaining the imposed subsystem constraints.

\paragraph{Heat and Fisher--KPP.}
We Fourier-transform the backbone proposal, replace the coefficients of the selected subsystem modes with the subsystem forecasts, and apply the inverse transform. All other coefficients remain unchanged at this stage. For FNO, the resulting field is used directly as the corrected state. This minimally invasive choice retains the backbone's prediction outside the selected modes and is naturally aligned with FNO's explicit spectral parameterization.

For ResCNN, we augment spectral replacement with a learned residual to adjust the remaining field in response to subsystem guidance. The correction network takes the current field, the subsystem-aligned proposal, and the subsystem forecasts as inputs, together with the forcing field when present, and outputs a residual field. We remove the selected Fourier components from this residual before adding it to the aligned proposal to form a corrected candidate. This extends the correction to the remaining components without changing the candidate's subsystem-imposed coefficients.

\paragraph{BVE.}
The 3 degree-1 spherical-harmonic variables constrain the solid-body rotational component, but do not determine the remaining vorticity structure. To adjust this remaining structure as well, both BVE backbones combine spectral matching with a learned residual correction. We first replace the degree-1 components of the backbone proposal with the subsystem forecasts. A convolutional correction network takes the current vorticity field, the original proposal, and the subsystem-aligned proposal as inputs and outputs a vorticity-field residual. Its degree-1 components are removed before addition to the aligned proposal, preserving the imposed subsystem coefficients.

\paragraph{Alanine.}
The subsystem describes overall translation, so matching its forecasts need not alter the molecule's internal configuration. After the backbone's constraint projection, we uniformly shift all atomic positions and velocities to match the predicted center-of-mass position and velocity. These shifts preserve the relative positions and velocities between atoms and require no learned correction network.

\subsubsection{Subsystem-only reconstruction}
\label{app:subsystem_only_decoders}

These decoders specify how the independently predicted subsystem variables are converted into full-state estimates for the subsystem-only diagnostic. 

\paragraph{Spectral reconstruction.}
For Heat, Fisher--KPP, and BVE, the predicted spectral coefficients are transformed back to physical space, with all unretained coefficients set to zero. The resulting reconstruction therefore contains only the selected spectral components. 

\paragraph{Molecular reconstruction.}
For Alanine, the center-of-mass variables do not uniquely specify the internal conformation; a learned decoder maps them to internal positions and velocities. These outputs are mass-centered before adding the predicted center-of-mass components, ensuring that the reconstruction retains the subsystem forecasts. The decoder is fitted on training trajectories.

\subsection{Numerical configurations}
\label{app:implementation_matrix}

\paragraph{Simulation settings.}
The discretization and time intervals used to generate the data are summarized in Table~\ref{tab:simulation_settings}. The prediction interval is the time between consecutive emulator states.

\begin{table}[H]
\caption{Simulation settings and emulator prediction intervals.}
\label{tab:simulation_settings}
\centering
\small
\setlength{\tabcolsep}{4pt}
\begin{tabularx}{\linewidth}{@{}lXlll@{}}
\toprule
System & Spatial representation & Data generation & \shortstack{Integration\\step} & \shortstack{Prediction\\interval} \\
\midrule
Heat & 256 grid points & Custom solver & 0.05 & 0.05 \\
Fisher--KPP & 2048 reference points; 512 emulator points & Custom solver & 0.00125 & 0.02 \\
BVE & $32\times64$ spherical grid & SpeedyWeather.jl & 900 s & 3 hours \\
Alanine & 22 atoms & OpenMM & 0.1 fs & 0.01 ps \\
\bottomrule
\end{tabularx}
\end{table}

Heat uses $L=2\pi$ and $\nu=0.02$, with time-dependent external forcing. Fisher--KPP uses $L=128$ and $D=r=1$, initialized from localized smooth profiles. Alanine uses the AMBER ff14SB force field \citep{maier2015ff14sb} with the Onufriev--Bashford--Case model II (OBC2) for implicit solvation \citep{onufriev2004obc}. Time is expressed in simulation units for Heat and Fisher--KPP.

\paragraph{Dataset splits.}
The numbers of training and development trajectories are listed in Table~\ref{tab:benchmark_splits}. Normalization statistics are computed from the training split only, and comparisons use the same development trajectories and initial states.

\begin{table}[H]
\caption{Numbers of training and development trajectories.}
\label{tab:benchmark_splits}
\centering
\small
\setlength{\tabcolsep}{8pt}
\begin{tabular}{@{}lcccc@{}}
\toprule
Split & Heat & Fisher--KPP & BVE & Alanine \\
\midrule
Training & 64 & 64 & 128 & 96 \\
Development & 12 & 12 & 32 & 32 \\
\bottomrule
\end{tabular}
\end{table}

\FloatBarrier

\input{sections/appendix_d_additional}

%% file: sections/appendix_d_additional.tex
\subsection{Additional results}
\label{app:complete_error_tables}

\subsubsection{Prediction performance}

\paragraph{Numerical results for diffusion benchmarks.} Table~\ref{tab:diffusion_endpoint_results} supplements Fig.~\ref{fig:diffusion_main} with endpoint NRMSE values and sample SDs for the Heat and Fisher--KPP rollouts.

\begin{table}[H]
\caption{Full-state endpoint NRMSE for Heat and Fisher--KPP at selected rollout steps (mean $\pm$ sample SD).}
\label{tab:diffusion_endpoint_results}
\centering
\footnotesize
\setlength{\tabcolsep}{2.2pt}
\renewcommand{\arraystretch}{0.93}
\resizebox{\linewidth}{!}{%
\begin{tabular}{@{}lrrrr@{}}
\toprule
Method & \multicolumn{4}{c}{Endpoint NRMSE at the indicated rollout step} \\
\midrule
Heat/ResCNN & $h=1$ & $h=100$ & $h=200$ & $h=500$ \\
Base   & $.00949\pm.00358$ & $.00602\pm.00130$ & $.00739\pm.00235$ & $.00964\pm.00353$ \\
HINE   & $.00947\pm.00357$ & $.00583\pm.00132$ & $.00727\pm.00233$ & $.00944\pm.00327$ \\
CasEm  & $.00938\pm.00357$ & $\mathbf{.00353\pm.00100}$ & $\mathbf{.00376\pm.00124}$ & $\mathbf{.00415\pm.00157}$ \\
Oracle & $.00938\pm.00357$ & $.00352\pm.00101$ & $.00376\pm.00123$ & $.00414\pm.00155$ \\
SO     & $.408\pm.152$ & $.305\pm.166$ & $.195\pm.109$ & $.138\pm.0843$ \\
\cmidrule(lr){1-5}
Heat/FNO & $h=1$ & $h=100$ & $h=500$ & $h=16000$ \\
Base   & $(2.32\pm.680)\!\times10^{-4}$ & $(6.42\pm1.49)\!\times10^{-4}$ & $.00132\pm.000651$ & $.0874\pm.0820$ \\
HINE   & $(2.31\pm.680)\!\times10^{-4}$ & $(6.30\pm1.40)\!\times10^{-4}$ & $.00118\pm.000591$ & $.0794\pm.0802$ \\
CasEm  & $(2.23\pm.639)\!\times10^{-4}$ & $\mathbf{(1.31\pm.459)\!\times10^{-4}}$ & $\mathbf{(8.92\pm4.00)\!\times10^{-5}}$ & $\mathbf{(1.32\pm.594)\!\times10^{-4}}$ \\
Oracle & $(2.23\pm.639)\!\times10^{-4}$ & $(1.31\pm.459)\!\times10^{-4}$ & $(8.90\pm4.01)\!\times10^{-5}$ & $(1.27\pm.580)\!\times10^{-4}$ \\
SO     & $.408\pm.152$ & $.305\pm.166$ & $.138\pm.0843$ & $.0915\pm.0746$ \\
\cmidrule(lr){1-5}
Fisher--KPP/ResCNN & $h=1$ & $h=100$ & $h=500$ & $h=1000$ \\
Base   & $.00333\pm.00145$ & $.0271\pm.00257$ & $.0449\pm.0223$ & $(8.31\pm17.3)\!\times10^{3}$ \\
HINE   & $.00331\pm.00147$ & $.0167\pm.00622$ & $.0353\pm.0181$ & $(3.38\pm7.63)\!\times10^{3}$ \\
CasEm  & $\mathbf{.00169\pm.000862}$ & $\mathbf{.0120\pm.00549}$ & $\mathbf{.0134\pm.00302}$ & $\mathbf{.443\pm.374}$ \\
Oracle & $.00170\pm.00102$ & $.00215\pm.000693$ & $.00242\pm.00122$ & $.490\pm.494$ \\
SO     & $.178\pm.0242$ & $.0339\pm.00736$ & $.0139\pm.00294$ & $.0105\pm.00209$ \\
\cmidrule(lr){1-5}
Fisher--KPP/FNO & $h=1$ & $h=100$ & $h=500$ & $h=1000$ \\
Base   & $.00340\pm.00136$ & $.0363\pm.0174$ & $.578\pm.121$ & $744\pm241$ \\
HINE   & $.00390\pm.00159$ & $.0314\pm.0137$ & $.787\pm.178$ & $(2.52\pm.793)\!\times10^{3}$ \\
CasEm  & $.00334\pm.00140$ & $\mathbf{.0126\pm.00538}$ & $\mathbf{.0154\pm.00283}$ & $\mathbf{.0181\pm.00628}$ \\
Oracle & $.00329\pm.00140$ & $.00391\pm.000503$ & $.00775\pm.00136$ & $.0147\pm.00744$ \\
SO     & $.178\pm.0242$ & $.0339\pm.00736$ & $.0139\pm.00294$ & $.0105\pm.00209$ \\
\bottomrule
\end{tabular}}
\end{table}

\subsubsection{Subsystem selection}
\label{app:frequency_band_ablation}

\paragraph{Subsystem size.} Tables~\ref{tab:subsystem_size_heat} and~\ref{tab:subsystem_size_fisher} provide the numerical results underlying Fig.~\ref{fig:subsystem_analysis}(a,b), with $K$ denoting the highest retained Fourier mode. The FNO backbone is fixed across subsystem sizes, and the Fisher--KPP subsystem models share a common initialization.

\begin{table}[H]
\caption{Heat/FNO subsystem-size ablation at 1000 steps. Full-state NRMSE ($\times10^{-3}$; mean $\pm$ sample SD) over 12 development trajectories.}
\label{tab:subsystem_size_heat}
\label{tab:subsystem_size_complete}
\centering
\small
\setlength{\tabcolsep}{5pt}
\begin{tabular}{@{}crr@{}}
\toprule
$K$ & CasEm & Oracle \\
\midrule
Base & \multicolumn{2}{c}{$1.616\pm0.619$} \\
0 & $1.049\pm0.569$ & $1.049\pm0.569$ \\
1 & $.4418\pm0.208$ & $.4418\pm0.208$ \\
2 & $.1868\pm0.106$ & $.1868\pm0.106$ \\
3 & $.1135\pm0.0641$ & $.1135\pm0.0641$ \\
4 & $.1054\pm0.0630$ & $.1054\pm0.0630$ \\
5 & $.07686\pm0.0390$ & $.07685\pm0.0389$ \\
6 & $.06798\pm0.0325$ & $.06799\pm0.0325$ \\
\bottomrule
\end{tabular}
\end{table}

\begin{table}[H]
\caption{Fisher--KPP/FNO subsystem-size ablation at 1000 steps. Values are full-state NRMSE (mean $\pm$ sample SD) over 12 development trajectories.}
\label{tab:subsystem_size_fisher}
\centering
\small
\setlength{\tabcolsep}{5pt}
\begin{tabular}{@{}crr@{}}
\toprule
$K$ & CasEm & Oracle \\
\midrule
Base & \multicolumn{2}{c}{$748.98\pm244$} \\
4  & $175.95\pm98.4$ & $199.03\pm95.8$ \\
8  & $15.075\pm6.83$ & $16.286\pm8.15$ \\
16 & $.01911\pm.00602$ & $.01476\pm.00741$ \\
32 & $.06342\pm.00848$ & $.01188\pm.00681$ \\
\bottomrule
\end{tabular}
\end{table}

\paragraph{Frequency selection.} Complementing the subsystem-size analysis in Section~\ref{sec:subsystem_choice}, we compare which Fourier modes are retained. Table~\ref{tab:frequency_band} shows that the low-frequency band $[0,8)$ outperforms both equal-width and wider high-frequency bands. The spectral error distributions in Figure~\ref{fig:error_spectra} further show that Base error increasingly concentrates in low-frequency components during rollout, explaining why correcting these components is particularly effective.

\begin{table}[H]
\caption{Mean full-state NRMSE at 1000 steps for Fisher--KPP/FNO with different Fourier subsystems.}
\label{tab:frequency_band}
\centering
\small
\setlength{\tabcolsep}{5pt}
\begin{tabular}{@{}lrrrrrr@{}}
\toprule
 & Base & $[0,8)$ & $[8,16)$ & $[16,24)$ & $[16,32)$ & Joint $0{:}16$ \\
\midrule
NRMSE & $744.066$ & $27.616$ & $55.751$ & $290.060$ & $282.600$ & $\mathbf{0.0181}$ \\
\bottomrule
\end{tabular}
\end{table}

\begin{figure}[H]
\centering
\includegraphics[width=\linewidth]{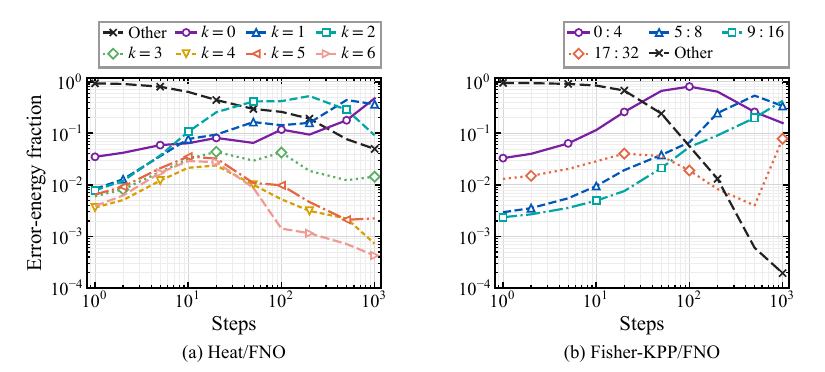}
\caption{Spectral fractions of Base error for Heat/FNO and Fisher--KPP/FNO. Colors distinguish individual modes for Heat and disjoint frequency bands for Fisher--KPP.}
\label{fig:error_spectra}
\end{figure}

\FloatBarrier

\subsection{Additional experimental example}
\label{app:central_force}

We use central-force orbital dynamics as an additional example to assess CasEm with a nonlinear aggregate readout and both MLP and Hamiltonian neural network (HNN) backbones.

\paragraph{Dynamics and subsystem.}
The central-force benchmark has a 6-dimensional full state $x=(q,v)\in\mathbb R^6$ and describes 3-dimensional orbital motion under an attractive inverse-square force and an additional radial input:
\begin{equation*}
\dot q=v,
\qquad
\mu\dot v=-\frac{\kappa}{\|q\|^3}q
+u(t)\frac{q}{\|q\|},
\qquad q,v\in\mathbb R^3.
\end{equation*}
Here, $q$ and $v$ are position and velocity, $\mu$ is the mass parameter, $\kappa$ controls the attractive force strength, and $u(t)$ is the signed radial input force.
For $r>0$, the 3-dimensional subsystem state is defined by
\begin{equation*}
z=(r,p_r,\ell)\in\mathbb R^3,
\qquad
r=\|q\|,
\qquad
p_r=\frac{q^\top v}{\|q\|},
\qquad
\ell=\|q\times v\|.
\end{equation*}
Here, $r$ is radial distance, $p_r$ denotes radial velocity, and $\ell$ is specific angular-momentum magnitude. They satisfy
\begin{equation*}
\dot r=p_r,\qquad
\dot p_r=\frac{\ell^2}{r^3}
-\frac{\kappa}{\mu r^2}+\frac{u(t)}{\mu},
\qquad
\dot\ell=0.
\end{equation*}
Thus, the subsystem captures radial orbital evolution without retaining the full spatial orientation. It is exactly closed, while its aggregation map from $(q,v)$ to $(r,p_r,\ell)$ is nonlinear.

\paragraph{Experimental configuration.} We use $\mu=\kappa=1$, randomly oriented initial orbits, and a time-dependent radial input. A custom solver generates 3-dimensional position--velocity trajectories with both integration and prediction intervals of $0.01$ simulation time units. The dataset contains 128 training and 24 development trajectories; normalization statistics are computed from the training split, and methods are compared on the same development trajectories and initial states. For HINE, fixed pooling of position and velocity components defines 4- and 2-dimensional coarse states. Their predicted future values guide a learned residual correction of the backbone prediction. Subsystem-only rollouts use a decoder fitted on training trajectories to reconstruct position and velocity from $(r,p_r,\ell)$.

\paragraph{Full-state correction.} Nonlinear aggregate targets are imposed by separating orbital direction from radial and tangential magnitudes. For a candidate state $(q,v)$, let $e_r=q/\|q\|$ be the radial unit vector and $e_\tau$ the unit direction of $v-(v^\top e_r)e_r$. Given the subsystem target $(r^\star,p_r^\star,\ell^\star)$, the state is reconstructed as
\begin{equation*}
q'=r^\star e_r,
\qquad
v'=p_r^\star e_r+\frac{\ell^\star}{r^\star}e_\tau.
\end{equation*}
This operation retains the candidate radial and tangential directions while matching the predicted radius, radial velocity, and specific angular-momentum magnitude.

With MLP, the backbone proposal is first reconstructed in this way. A correction network takes the current full state, external input, reconstructed proposal, and subsystem target as inputs and predicts a 6-dimensional position--velocity residual. The residual is added to the reconstructed proposal, and the same target-matching operation is applied again.

With HNN, the low-dimensional predictor supplies the next aggregate state and an initial angular increment. Learned phase-correction and fusion modules use the subsystem state, external input, and aggregate discrepancies and relative phases involving the HNN proposal to adjust this angle within the orbital plane. Position and velocity are then reconstructed from the updated radial and tangential directions using the same subsystem targets.

\paragraph{Results and analysis.}
Table~\ref{tab:central_force_results} shows that CasEm
substantially outperforms the baselines with both backbones,
suppressing the MLP backbone's explosive error growth and
improving prediction accuracy with HNN.
Compared with the simpler closed subsystems in the main-text
examples, the central-force subsystem provides more limited
simplification of the original system in both dimension and
dynamical complexity: it reduces the state dimension from
6 to 3 while retaining nonlinear radial dynamics.
Errors in the predicted radius or angular momentum can
affect orbital angular speed and accumulate into phase
errors in the full-state trajectory.
This retained dynamical complexity and its influence on
orbital phase may help explain the larger gap between
CasEm and Oracle, despite the substantial gains over the
baselines. These results demonstrate that CasEm remains
effective with nonlinear aggregate readouts even when
the subsystem offers only modest simplification.

\begin{table}[H]
\caption{Full-state endpoint NRMSE for central-force dynamics (mean $\pm$ sample SD).}
\label{tab:central_force_results}
\centering
\small
\setlength{\tabcolsep}{2.6pt}
\renewcommand{\arraystretch}{0.96}
\resizebox{\linewidth}{!}{%
\begin{tabular}{@{}llrrrr@{}}
\toprule
Backbone & Method & $h=1$ & $h=200$ & $h=1000$ & $h=2000$ \\
\midrule
MLP & Base  & $(9.82\pm7.03)\!\times10^{-4}$ & $.390\pm.326$ & $1.32\pm.491$ & $(1.24\pm6.08)\!\times10^{36}$ \\
& HINE      & $(9.82\pm7.03)\!\times10^{-4}$ & $.392\pm.329$ & $1.36\pm.560$ & $(1.36\pm6.68)\!\times10^{27}$ \\
& CasEm     & $\mathbf{(5.58\pm3.74)\!\times10^{-4}}$ & $\mathbf{.0729\pm.0501}$ & $\mathbf{.312\pm.360}$ & $\mathbf{.595\pm.481}$ \\
& Oracle    & $(5.33\pm3.73)\!\times10^{-4}$ & $.0615\pm.0372$ & $.172\pm.112$ & $.338\pm.180$ \\
\cmidrule(lr){1-6}
HNN & Base  & $(1.56\pm.930)\!\times10^{-4}$ & $.0225\pm.0157$ & $.120\pm.0936$ & $.271\pm.209$ \\
& HINE      & $.0170\pm.00635$ & $1.33\pm.500$ & $1.27\pm.250$ & $1.11\pm.234$ \\
& CasEm     & $\mathbf{(4.88\pm5.58)\!\times10^{-5}}$ & $\mathbf{.00965\pm.00913}$ & $\mathbf{.0633\pm.0970}$ & $\mathbf{.145\pm.139}$ \\
& Oracle    & $(4.24\pm6.98)\!\times10^{-5}$ & $.00601\pm.00663$ & $.0151\pm.0195$ & $.0292\pm.0398$ \\
\cmidrule(lr){1-6}
\multicolumn{2}{l}{SO (shared)} & $.877\pm.196$ & $.905\pm.201$ & $.907\pm.211$ & $.981\pm.288$ \\
\bottomrule
\end{tabular}}
\end{table}

\FloatBarrier

\subsection{Ablation of unidirectional coupling}
\label{app:unidirectional_ablation}

We assess the role of CasEm's unidirectional coupling by introducing a learned reverse connection from the full state to the subsystem forecast. Such a connection can, in principle, capture unresolved influences when aggregate dynamics are only approximately closed, but makes the guiding prediction explicitly dependent on the full-state rollout, thereby breaking the one-way cascade structure.

\begin{figure}[H]
\begin{minipage}[c]{0.54\linewidth}
\paragraph{Experimental setup.} We consider Heat and Fisher--KPP with ResCNN and FNO backbones. 
The two systems provide contrasting closure settings.
For Heat, the retained Fourier modes evolve independently
of discarded modes, making full-state information unnecessary
for exact aggregate evolution.
In Fisher--KPP, the nonlinear reaction couples retained
and discarded modes, so the reverse connection may convey
information absent from the subsystem state.
We augment the subsystem update in Eq.~(\ref{eq:casem_rollout}) with a learned full-state residual, yielding the bidirectionally coupled recurrence illustrated in Figure~\ref{fig:feedback_architecture}: \begin{equation} \left\{ \begin{aligned} \widehat z_{t+1} &=g_\psi(\widehat z_t,u_t)+R_\omega(\widehat x_t,\widehat z_t,u_t),\\ \widehat x_{t+1} &=\mathcal C_\phi\!\left(f_\theta(\widehat x_t,u_t),\widehat z_{t+1}\right). \end{aligned} \right. \label{eq:feedback_ablation} \end{equation} where $R_\omega$ is implemented by a lightweight 1D periodic
convolutional network followed by the subsystem's aggregation
map. Its output is a residual correction to the subsystem's
aggregate variables, i.e., the retained Fourier coefficients.
The forcing input $u_t$ is omitted for unforced dynamics.
The backbone, subsystem predictor, correction map, and rollout settings remain fixed. Only the residual network is fine-tuned, using aggregate-prediction losses over progressively longer rollout windows.
\end{minipage}\hfill
\begin{minipage}[c]{0.43\linewidth}
\centering
\includegraphics[width=\linewidth]{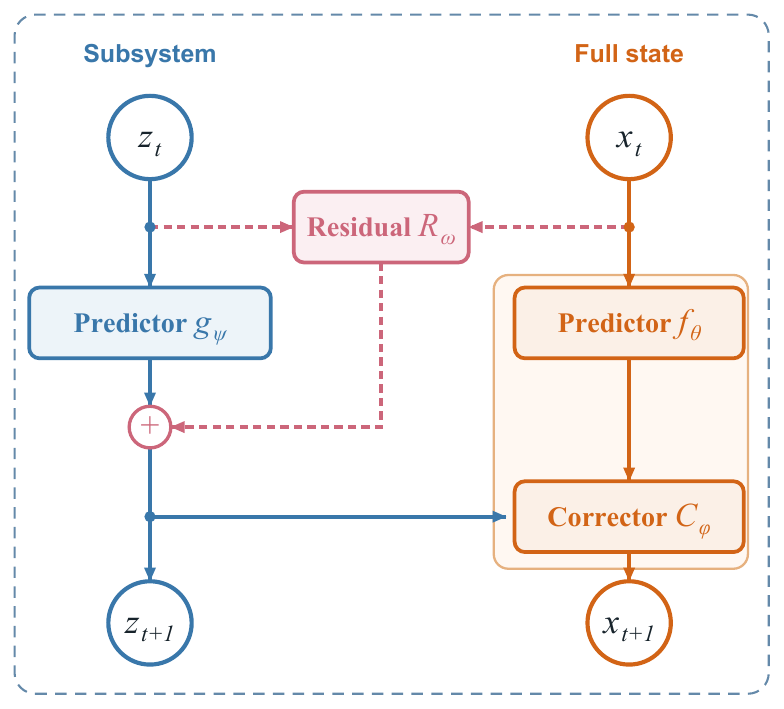}
\caption{One-step bidirectional variant of Figure~\ref{fig:casem_architecture}(c). Pink elements denote the added residual network and connections. Hats and forcing inputs are omitted.}
\label{fig:feedback_architecture}
\end{minipage}
\end{figure}

\paragraph{Results and analysis.} Figure~\ref{fig:feedback_ablation} shows that bidirectional coupling does not improve long-horizon accuracy over CasEm in any of the four configurations. Full-state errors increase at longer horizons and can even exceed those of Base, with generally larger variability across trajectories. Even for Fisher--KPP with FNO, where the feedback variant closely matches CasEm over intermediate horizons, its errors still grow substantially at longer horizons. The bottom row further shows pronounced deterioration in aggregate-component accuracy. This deterioration is consistent with error amplification through the added reverse connection: full-state errors can contaminate aggregate forecasts, which then affect subsequent full-state predictions. Together, these results support the importance of CasEm's unidirectional coupling for maintaining accurate and consistent long-horizon predictions in these benchmarks.

\begin{figure}[t]
\centering
\includegraphics[width=\linewidth]{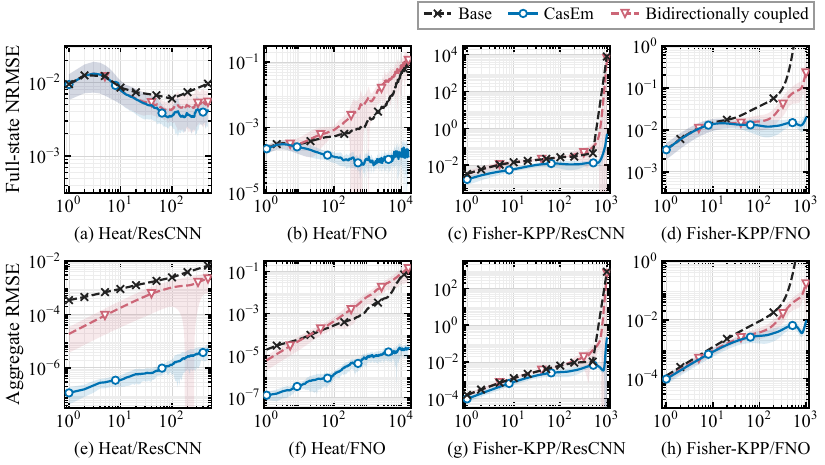}
\caption{Ablation of unidirectional coupling on Heat and Fisher--KPP. Top: full-state NRMSE. Bottom: aggregate-component RMSE of the full-state predictions. Lines show means over 12 development trajectories; shaded bands show one sample SD for CasEm and the bidirectionally coupled variant. Horizontal axes show rollout steps.}
\label{fig:feedback_ablation}
\end{figure}

%% file: sections/climate_results_appendix.tex
\input{sections/appendix_e_restructured}

%% file: sections/appendix_e_restructured.tex

\section{Climate Emulation}
\label{app:climate_appendix}

This appendix supplements Section~\ref{sec:climate_emulation} with subsystem details, evaluation protocols, metric definitions, and additional results.

\subsection{Subsystem details}
\label{app:climate_details}

The water and energy subsystems introduced in Section~\ref{sec:climate_emulation} are specified below, including their aggregate variables, independent dynamics, and full-state corrections.

\subsubsection{Water subsystem}

\paragraph{Aggregate variables.} Let $x^{\mathrm{water}}_{k,\ell}(s)$ denote the total-water mixing ratio at grid point $s$, vertical layer $\ell$, and subsystem time $t_k$, with $k$ indexing 4-day updates. For a target region $R_r$ containing $|R_r|$ grid points, the regional aggregate is
\begin{equation*}
 z_{k,\ell,r}=|R_r|^{-1}\sum_{s\in R_r}x^{\mathrm{water}}_{k,\ell}(s).
\end{equation*}
It summarizes regional moisture and its vertical distribution. Each of the 8 water layers is reduced by arithmetic $15\!\times\!15$ block means to a $12\!\times\!24$ subsystem target grid. 

\paragraph{Independent dynamics.}
A background $b(t)$, consisting of a linear trend and 4 annual harmonics, is fitted to the mean of the fitting trajectories. The 4 lowest 2-dimensional discrete cosine transform (DCT) modes ($2\!\times\!2$) are retained per layer. Let $D_{\mathrm{DCT}}$ denote the projection onto the retained
DCT modes, and let $D_{\mathrm{DCT}}^\dagger$ reconstruct the
residual field by zero-filling the omitted modes before inversion.
The vector $a_k$ denotes the model's retained DCT coefficient state for the background-subtracted regional water fields across all 8 layers.

A general linear model for the coefficient dynamics uses a transition matrix $A_{\mathrm w}$ and an input matrix $B_{\mathrm w}$.
The prescribed exogenous input vector $u_k$ is defined on the
4-day subsystem time grid and may include solar irradiance
and sea-surface temperature available during forecasting.
These inputs are not obtained from the full-state rollout.
Writing $[\cdot]_+$ for componentwise nonnegative clipping,
the coefficient update and reconstructed water forecast are
\[
a_{k+1}=A_{\mathrm w}a_k+B_{\mathrm w}u_k,
\qquad
\widehat z_k=
\bigl[b(t_k)+D_{\mathrm{DCT}}^\dagger a_k\bigr]_+.
\]
When appropriate, $A_{\mathrm w}$ may be constrained to be
Schur stable, with all eigenvalues strictly inside the unit disk.
The initial coefficient vector encodes the difference between the initial target and $b(t_0)$. The predictor has no full-state feedback and clips reconstructed water to be nonnegative. Figure~\ref{fig:water_sdy_workflow}(a) summarizes this construction.

\paragraph{Full-state correction.}
S-DY predicts 6 consecutive states at 6-hour intervals per autoregressive round, advancing 1.5 days. The water subsystem independently predicts the next aggregate state 4 days ahead. Corrections are applied every 12 days, corresponding to 8 S-DY rounds and 3 subsystem steps (Fig.~\ref{fig:water_sdy_workflow}(b)). Subsystem targets are lifted to the native grid, and target--proposal mismatches are aggregated over correction regions using cosine-area weights. Layer-specific regional mixing-ratio increments use a gain of $0.5$ and are capped at $\pm5\times10^{-4}\,\mathrm{kg\,kg^{-1}}$. No non-water field is changed directly.

\begin{figure}[!t]
\centering
\includegraphics[width=\linewidth]{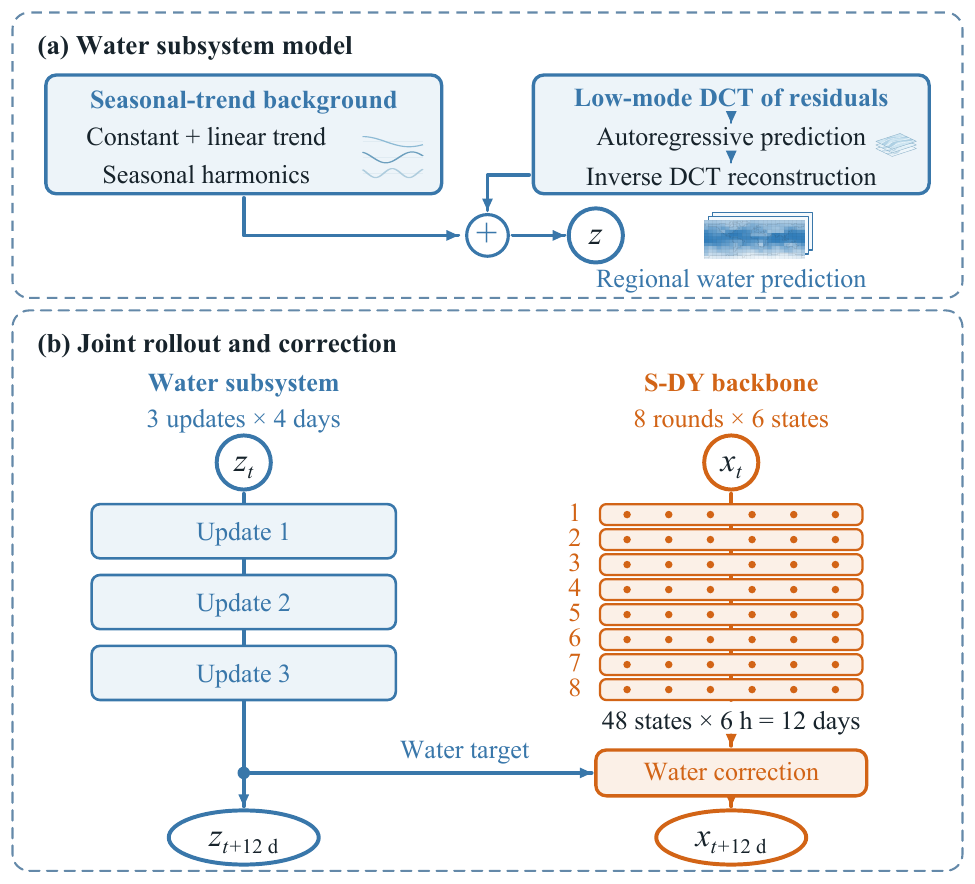}
\caption{Water-subsystem construction and one-way guidance of S-DY. (a) Background-subtracted water fields undergo low-mode DCT encoding, autoregressive prediction, and inverse DCT reconstruction. Restoring the fitted seasonal--trend background yields regional water predictions. (b) 3 subsystem updates of 4 days each match 8 S-DY rounds, each producing 6 states at 6-hour intervals. Every 12 days, the subsystem target corrects only the full-state water fields, without feedback to the subsystem. Here, $t$ denotes physical time.}
\label{fig:water_sdy_workflow}
\end{figure}

\subsubsection{Energy subsystem}

\paragraph{Aggregate variables.} The energy readout is global-mean column-integrated moist energy. With 4 frames per day at 6-hour intervals,
\begin{equation*}
 E(X_k)=\frac14\sum_{j=1}^{4}\sum_s w_s\sum_\ell
 \left(c_pT_{k,j,\ell}(s)+L_vx^{\mathrm{water}}_{k,j,\ell}(s)+g h_s\right)
 \frac{\Delta p_{k,j,\ell}(s)}{g},
\end{equation*}
where $w_s$ is normalized cosine-latitude weight and $h_s$ is surface height; kinetic energy is not included. Here $x^{\mathrm{water}}$, $T$, and $\Delta p$ are total-water mixing ratio, temperature, and layer pressure thickness; $c_p$ is the specific heat capacity, $L_v$ the latent heat of vaporization, and $g$ gravitational acceleration, and $X_k$ comprises the 4 proposal frames for day $k$, sampled at 6-hour intervals. 

\paragraph{Independent dynamics.}
A learned forced affine model evolves independently,
\begin{equation*}
 \widehat E_{k+1}=a\widehat E_k+b^\top u_k+c,
\end{equation*}
using forcing inputs $u_k$. It has no full-state feedback, and the fitted affine model is stable.

\paragraph{Full-state correction.}
Because global energy is a single scalar, the same energy mismatch can correspond to different temperature-correction patterns across regions and vertical levels. We therefore use a spatially conditioned diffusion U-Net to model the conditional distribution of these corrections. Conditioned on the energy mismatch, external forcing, and the backbone's temperature and pressure fields, it generates temperature increments for the 4 daily frames. Training uses paired frozen-ACE predictions and reference states from IC1--10, with their temperature differences as targets and reference-derived energy differences as conditions. The network learns to recover these residual fields from Gaussian-noise-corrupted versions and is further fine-tuned on prediction--reference pairs collected during corrected rollouts.
With gain $\eta$, the desired corrected readout is
\begin{equation*}
 E_{\mathrm{desired}}=E(X_{\mathrm{proposal}})+
 \eta\bigl(\widehat E-E(X_{\mathrm{proposal}})\bigr).
\end{equation*}
After the diffusion update, a spatially uniform temperature offset enforces this desired readout. No other fields are directly changed, and the final corrected frame is fed back to the backbone. For ACE-Energy, corrections are applied daily with gain $\eta=0.02$ and 20 denoising diffusion implicit model (DDIM) sampling steps.

\subsection{Evaluation protocol and metrics}
\label{app:climate_protocol}
\label{app:climate_metrics}

This section gives the mathematical definitions and aggregation rules for the metrics introduced in Section~\ref{sec:climate_emulation}, including the regional case study.

\paragraph{Data splits.} Following the training--validation split used by ACE and S-DY \citep{wattmeyer2023ace,ruhlingcachay2024sphericaldyffusion}, we fit the subsystem on IC1--10 and use IC11 for the main evaluation.

\subsubsection{Global rollout and ensemble metrics}
\label{app:climate_ensemble}

\paragraph{Global rollout errors.}
Let $T$ be the number of evaluated time points, excluding initialization, and $S$ the number of spatial grid points. For matched forecast trajectory $i$, field $v$, time $t$, and grid point $s$, define
\begin{equation*}
 e_{i,t,v}(s)=\widehat x_{i,t,v}(s)-x_{i,t,v}(s),
 \qquad
 \bar e_{i,v}(s)=T^{-1}\sum_t e_{i,t,v}(s),
\end{equation*}
and the unweighted reference spatiotemporal population variance
\begin{equation*}
 \mu_{i,v}=(TS)^{-1}\sum_{t,s}x_{i,t,v}(s),\qquad
 \sigma_{i,v}^{2}=(TS)^{-1}\sum_{t,s}(x_{i,t,v}(s)-\mu_{i,v})^2.
\end{equation*}
The metrics are defined as
\begin{equation*}
 \mathrm{ST}_{i,v}=
 \frac{\sqrt{T^{-1}\sum_t\sum_s w_s e_{i,t,v}(s)^2}}{\sigma_{i,v}},
 \qquad
 \mathrm{TM}_{i,v}=
 \frac{\sqrt{\sum_s w_s\bar e_{i,v}(s)^2}}{\sigma_{i,v}},
\end{equation*}
where $w_s=\cos(\operatorname{lat}_s)/\sum_{s'}\cos(\operatorname{lat}_{s'})$. Both ST and TM use cosine-latitude area weighting in the error numerator and are normalized by the reference spatiotemporal standard deviation. Group scores average the normalized variable-wise errors equally.

\paragraph{Energy-readout errors.} These are computed from the full-state rollout, rather than the independent subsystem forecast. For reference and predicted daily energies $E_t$ and $\widehat E_t$, respectively, $\mathrm{ST}=\mathrm{RMSE}(E_{1:T},\widehat E_{1:T})/\sigma_E$ and $\mathrm{TM}=|\operatorname{mean}_t(\widehat E_t-E_t)|/\sigma_E$, where $\sigma_E$ is the reference daily-energy standard deviation.

\paragraph{Ensemble scores.} For the evaluation trajectory, write $\sigma_v$ for the reference standard deviation defined above. Let $M$ be the ensemble size, $X_{m,t,v}(s)$ the forecast of field $v$ at time $t$ and grid point $s$ in member $m$, and $\overline{x}^{*}_{v}(s)$ the reference 10-year mean. The member and ensemble time-mean fields are
\begin{equation*}
 \overline X_{m,v}(s)=\frac1T\sum_{t=1}^{T}X_{m,t,v}(s),
 \qquad
 \overline X_v(s)=\frac1M\sum_{m=1}^{M}\overline X_{m,v}(s).
\end{equation*}
The normalized ensemble-mean error and fair CRPS are
\begin{equation*}
 \mathrm{NRMSE}_{\mathrm{ens},v}
 =\frac{\sqrt{\sum_s w_s\bigl(\overline X_v(s)-\overline{x}^{*}_{v}(s)\bigr)^2}}{\sigma_v},
\end{equation*}
\begin{equation*}
\begin{aligned}
 \mathrm{fCRPS}_v
 =\frac1{\sigma_v}\sum_s w_s\Bigg[
 &\frac1M\sum_{m=1}^{M}\left|\overline X_{m,v}(s)-\overline{x}^{*}_{v}(s)\right|\\
 &-\frac1{2M(M-1)}\sum_{\substack{m,n=1\\m\ne n}}^{M}
 \left|\overline X_{m,v}(s)-\overline X_{n,v}(s)\right|\Bigg].
\end{aligned}
\end{equation*}
The double sum uses distinct ordered member pairs. The reported S-DY scores use $M=25$ members and are averaged equally over the variables in each group.

\paragraph{Statistical summaries.} Member summaries are reported as member mean $\pm$ sample SD. For Fig.~\ref{fig:climate_annual_tm}, annual means and sample SDs are computed across the rollouts available in each calendar year: 1, 2, and 3 starts in 2021, 2022, and 2023, respectively, and 4 thereafter. The SD therefore describes variation across start dates on IC11.

\subsubsection{Spatial diagnostics and relative changes}
\label{app:climate_spatial_metrics}

\paragraph{Grouped spatial errors.} For method $a\in\{\mathrm{Base},\mathrm{CasEm}\}$, variable $j$, and ensemble member $m$, let $\overline{x}^{\,a}_{j,m}(s)$ denote the 10-year mean prediction at grid point $s$, and $\overline{x}^{*}_{j}(s)$ the corresponding reference mean. Define
\begin{align*}
e_{a,j}(s)&=\left[\frac1{25}\sum_{m=1}^{25}\left(\frac{\overline{x}^{\,a}_{j,m}(s)-\overline{x}^{*}_{j}(s)}{\sigma_j}\right)^2\right]^{1/2},
&E_{a,G}(s)&=\frac1{|G|}\sum_{j\in G}e_{a,j}(s),\notag\\
R_G(s)&=100\frac{E_{\mathrm{Base},G}(s)-E_{\mathrm{CasEm},G}(s)}{E_{\mathrm{Base},G}(s)}.
\end{align*}
Here $\sigma_j$ is the frozen per-variable target standard deviation, and $G$ is the 8 water layers, all 34 fields, or the 26 non-water fields. The square root is taken per variable before the equal average over variables. With $\omega_s=\cos(\operatorname{lat}_s)/\sum_{s'}\cos(\operatorname{lat}_{s'})$ and $\langle E_{a,G}\rangle_\omega=\sum_s\omega_sE_{a,G}(s)$, quoted group reductions are $100(1-\langle E_{\mathrm{CasEm},G}\rangle_\omega/\langle E_{\mathrm{Base},G}\rangle_\omega)$, not means of grid-cell percentages.

\paragraph{Relative reductions and aggregation.}
\phantomsection\label{app:climate_relative_changes}
For any error metric $E$, the relative reduction is
\[
\Delta_E
=
100\frac{E_{\mathrm{Base}}-E_{\mathrm{CasEm}}}
{E_{\mathrm{Base}}}.
\]

Across initial conditions, reductions are ratios of equal-trajectory mean errors,
\begin{equation*}
 100\left(1-\frac{\bar E^{\mathrm{CasEm}}}{\bar E^{\mathrm{Base}}}\right),
 \qquad \bar E^a=N_{\mathrm{IC}}^{-1}\sum_i E_i^a,
\end{equation*}
where $N_{\mathrm{IC}}$ is the number of evaluated initial-condition trajectories. These reductions are not means of per-trajectory percentages. Table~\ref{tab:climate_ensemble_scores} averages variable-wise relative reductions, rather than taking the ratio of group-mean errors used in Table~\ref{tab:climate_results}.

\paragraph{Ablation changes.} For TM error $E$, the relative change reported in Table~\ref{tab:climate_ablation_main} is
\begin{equation*}
 \Delta_{\mathrm{abl}}=100\frac{E_{\mathrm{abl}}-E_{\mathrm{full}}}{E_{\mathrm{full}}},
\end{equation*}
where $E_{\mathrm{full}}$ is the corresponding full-CasEm error and positive values indicate degradation.

For the horizontal correction-field comparisons in Table~\ref{tab:climate_horizontal_transforms}, define the relative improvement over Base and its reduction under ablation as
\begin{equation*}
 I(E)=100(1-E/E_{\mathrm{Base}}),\qquad
 L(E)=I_{\mathrm{full}}-I_{\mathrm{abl}}
 =100(E_{\mathrm{abl}}-E_{\mathrm{full}})/E_{\mathrm{Base}}.
\end{equation*}
Here $L$ is a loss of improvement in percentage points (pp), rather than the relative error change reported in Table~\ref{tab:climate_ablation_main}.

\subsubsection{Regional case-study evaluation}
\label{app:climate_case}

The following setup and metrics are used in Section~\ref{sec:climate_case} and Table~\ref{tab:case_south_asia_gain}.

\paragraph{Study domain.} Figure~\ref{fig:case_south_asia_location} shows the South Asian study domain.

\begin{figure}[!htbp]\centering
\includegraphics[width=0.5\linewidth]{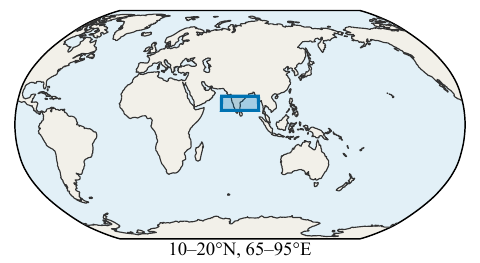}
\caption{South Asian case-study domain, outlined by the rectangle.}
\label{fig:case_south_asia_location}
\end{figure}

\paragraph{Daily fields and regional means.}
Daily TWP is obtained by averaging 4 frames at 6-hour intervals. Let $T_m(t,j)$ denote daily TWP for trajectory $m$ at grid cell $j$. Its regional mean is
\[
X_m(t)=\sum_j w_j T_m(t,j),
\qquad
w_j=\frac{\cos\phi_j}{\sum_k\cos\phi_k},
\]
where the sums run over cells in the domain and $\phi_j$ is latitude.

\paragraph{Slow-anomaly construction.}
For slow-anomaly diagnostics, we fit an intercept, a linear trend, and the first 4 sine--cosine harmonics of a 365-day annual cycle to each trajectory's continuous 10-year regional-mean sequence. Denoting the residual by $r_m(t)$, the slow anomaly is
\[
S_m(t)=\operatorname{MA}_{31}[r_m](t)-\operatorname{MA}_{127}[r_m](t).
\]
Both moving averages are centered.
Trimming the ends of the 3650-day sequence to accommodate
the complete 127-day window leaves 3524 valid days.
We then select June--August, retaining 10 complete summers
of 92 days each ($N=920$). Filtering thus precedes seasonal selection.

All regional metrics use the same June--August date set $\mathcal I$, with $N=|\mathcal I|$. The subscript $\mathrm{ref}$ denotes the reference trajectory. Only the amplitude, exceedance-frequency, and regional slow-RMSE diagnostics use the slow-anomaly series $S_m(t)$.

\paragraph{Slow-anomaly amplitude.}
Slow-anomaly amplitude is measured by the population standard deviation,
\[
\sigma_m=
\left[
\frac{1}{N}\sum_{t\in\mathcal I}
\bigl(S_m(t)-\overline S_m\bigr)^2
\right]^{1/2}.
\]
Here, $\overline S_m=N^{-1}\sum_{t\in\mathcal I}S_m(t)$. The amplitude error is
\[
E_{\mathrm{amp},m}=|\sigma_m-\sigma_{\mathrm{ref}}|.
\]

\paragraph{Exceedance frequency.}
For $p=95$, let $\tau_p$ be the $p$th percentile of $|S_{\mathrm{ref}}|$ over $\mathcal I$. The exceedance frequency and its absolute bias are
\[
f_{p,m}
=
\frac{1}{N}\sum_{t\in\mathcal I}
\mathbf 1\!\left(|S_m(t)|>\tau_p\right),
\qquad
E_{p,m}=100\,|f_{p,m}-f_{p,\mathrm{ref}}|.
\]
Actual frequencies are reported as percentages and frequency biases in percentage points (pp).

\paragraph{Regional and native-grid errors.}
Regional slow RMSE is
\[
R_{\mathrm{regional},m}^{\mathrm{slow}}
=
\left[
\frac{1}{N}\sum_{t\in\mathcal I}
\bigl(S_m(t)-S_{\mathrm{ref}}(t)\bigr)^2
\right]^{1/2}.
\]
Native-grid daily RMSE uses unfiltered daily fields and is
\[
R_{\mathrm{grid},m}^{\mathrm{daily}}
=
\left[
\frac{1}{N}\sum_{t\in\mathcal I}\sum_j
w_j\bigl(T_m(t,j)-T_{\mathrm{ref}}(t,j)\bigr)^2
\right]^{1/2}.
\]
Regional RMSE measures errors after spatial averaging, whereas native-grid RMSE retains spatially varying errors.

\paragraph{Transport, convergence, and wind.}
The additional transport, convergence, and wind diagnostics use unfiltered daily means over the same 920 summer days. At each six-hourly time, column-integrated total-water transport is $\mathbf F=\sum_{\ell=1}^{8}x^{\mathrm{water}}_\ell\Delta p_\ell(u_\ell,v_\ell)/g$, where $x^{\mathrm{water}}_\ell$ is total-water mixing ratio, $\Delta p_\ell$ is layer pressure thickness, and $u_\ell$ and $v_\ell$ are the zonal and meridional wind components, respectively; convergence is $C=-\nabla_h\cdot\mathbf F$, reported in kg m$^{-2}$ day$^{-1}$. Here, $\nabla_h\cdot$ denotes horizontal divergence. We average these diagnostic fields and the wind components over each day's four frames, then form area-weighted regional means before computing RMSE against the reference. Vector errors sum squared errors of the two horizontal components without dividing by two; wind RMSE additionally pools these squared errors equally over all eight model levels and days before taking the square root. No slow-anomaly filtering is applied.

\FloatBarrier
\subsection{Additional results}

\subsubsection{Robustness across trajectories and initialization times}
\label{app:climate_all11}

\paragraph{Cross-trajectory comparison.}
Table~\ref{tab:climate_multi_ic} supplements the IC11 comparison in Table~\ref{tab:climate_results} with mean errors across all 11 trajectories for ACE and S-DY.

\begin{table}[!htbp]
\caption{10-year climate errors averaged equally across 11 trajectories. Parentheses give reductions relative to the corresponding backbone.}
\label{tab:climate_multi_ic}
\centering\small\setlength{\tabcolsep}{3pt}
\begin{tabular}{@{}llrrrr@{}}\toprule
Backbone & Method & Full TM & Full ST & Water TM & Water ST \\\midrule
S-DY & Base & 0.0786 & 0.8136 & 0.2096 & 0.9500 \\
& CasEm & \textbf{0.0444} (43.46\%) & \textbf{0.7879} (3.16\%) & \textbf{0.0756} (63.92\%) & \textbf{0.8447} (11.08\%) \\
ACE & Base & 0.1353 & 0.8967 & 0.3699 & 1.2280 \\
& CasEm & \textbf{0.0502} (62.86\%) & \textbf{0.7903} (11.86\%) & \textbf{0.0804} (78.26\%) & \textbf{0.8693} (29.21\%) \\
\bottomrule\end{tabular}
\end{table}

\phantomsection\label{app:climate_annual}
\paragraph{Initialization-time comparison.} Figure~\ref{fig:climate_annual_by_start} shows annual full-state and total-water ST/TM error curves separately for the 2021--2024 starts on IC11, supplementing the means and SDs in Fig.~\ref{fig:climate_annual_tm}.

\begin{figure}[!htbp]
\centering
\includegraphics[width=\linewidth]{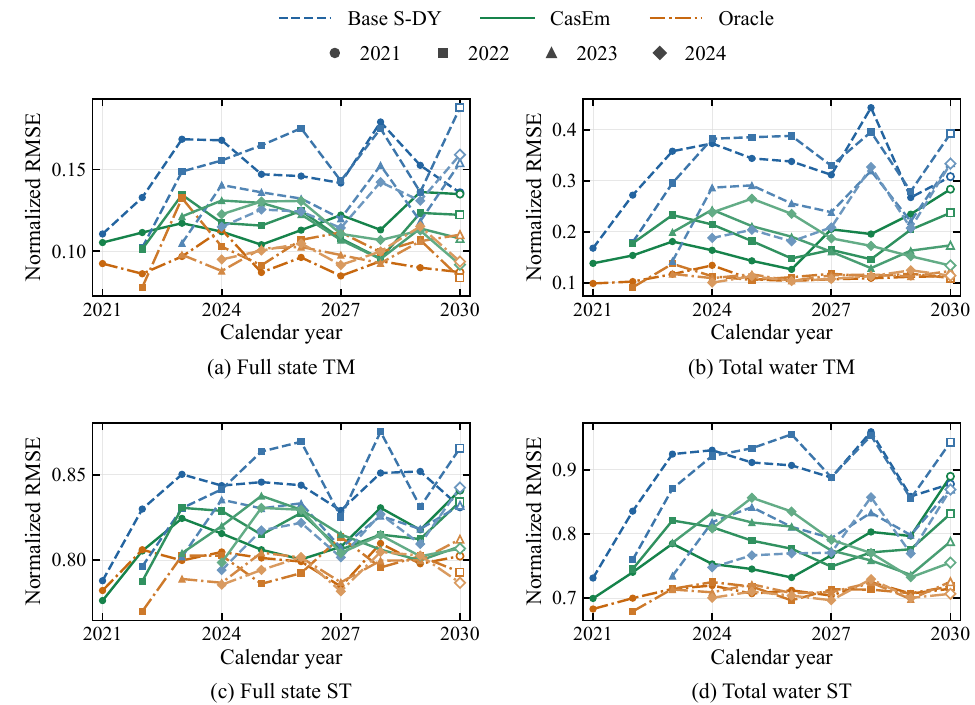}
\caption{Per-start annual full-state and total-water TM (a,b) and ST (c,d) NRMSE. Colors distinguish methods; markers distinguish 2021--2024 rollout starts on IC11.}
\label{fig:climate_annual_by_start}
\end{figure}

\FloatBarrier

\subsubsection{Energy-subsystem results}
\label{app:climate_energy}

Table~\ref{tab:ace_energy_summary} shows that energy-subsystem guidance improves both the energy readout and full-state predictions on IC11, with additional gains in total-water fields. All 6 reported errors also decrease when averaged equally across IC1--11. 
These results support energy aggregates as a potential subsystem choice. We acknowledge, however, that their benefits are less consistent across trajectories than those of water-subsystem guidance.
For S-DY, Base already predicts energy accurately (small $\rho$), leaving limited room for further gains from energy-subsystem guidance. We therefore focus on ACE here.
\begin{table}[!htbp]\centering\footnotesize
\setlength{\tabcolsep}{2pt}
\caption{ACE-Energy NRMSE on IC11 and averaged equally across IC1--11. Energy errors are computed from the full-state rollout. Parentheses give reductions relative to Base; reductions for IC1--11 use ratios of mean errors.}
\label{tab:ace_energy_summary}
\begin{tabular*}{\linewidth}{@{\extracolsep{\fill}}lrrrrrr@{}}\toprule
& \multicolumn{2}{c}{Full state} & \multicolumn{2}{c}{Total water} & \multicolumn{2}{c}{Energy} \\
\cmidrule(lr){2-3}\cmidrule(lr){4-5}\cmidrule(lr){6-7}
Method & ST & TM & ST & TM & ST & TM \\
\midrule
\multicolumn{7}{@{}l}{\textit{IC11}} \\
Base & 0.9409 & 0.1580 & 1.3832 & 0.4691 & 0.4710 & 0.1637 \\
CasEm & 0.8727 (7.2\%) & 0.1180 (25.3\%) & 1.1822 (14.5\%) & 0.3277 (30.1\%) & 0.4065 (13.7\%) & 0.0540 (67.0\%) \\
\midrule
\multicolumn{7}{@{}l}{\textit{IC1--11 mean}} \\
Base & 0.8967 & 0.1353 & 1.2280 & 0.3699 & 0.4873 & 0.0835 \\
CasEm & 0.8677 (3.2\%) & 0.1177 (13.0\%) & 1.1498 (6.4\%) & 0.3168 (14.4\%) & 0.3873 (20.5\%) & 0.0726 (13.1\%) \\
\bottomrule\end{tabular*}
\end{table}

\FloatBarrier
\subsubsection{Variable-wise and spatial improvements}
\label{app:climate_group_spatial}
To supplement the aggregate climate scores in Table~\ref{tab:climate_results}, Table~\ref{tab:climate_ensemble_scores} breaks down time-mean error reductions by physical variable group, and Fig.~\ref{fig:climate_group_tm_spatial} shows their spatial distribution. The spatial metrics and variable-wise aggregation are defined in Appendix~\ref{app:climate_spatial_metrics}.

\begin{table}[!htbp]\centering\footnotesize
\caption{10-year-mean-field error reductions across physical variable groups. Each cell gives the mean per-variable reduction and the number of improved variables.}
\label{tab:climate_ensemble_scores}
\begin{tabular}{@{}lccc@{}}\toprule
Group & Member RMSE & Ensemble-mean RMSE & Fair CRPS \\\midrule
Water & 18.49\% (8/8) & 16.11\% (6/8) & 16.45\% (5/8) \\
Temperature & 8.90\% (8/9) & 11.01\% (8/9) & 7.33\% (5/9) \\
Zonal wind & 7.25\% (8/8) & 8.00\% (8/8) & 8.11\% (6/8) \\
Meridional wind & 1.47\% (6/8) & 1.99\% (7/8) & 0.88\% (7/8) \\
Surface pressure & 5.38\% (1/1) & 9.75\% (1/1) & 9.19\% (1/1) \\\midrule
Other fields & 5.97\% (23/26) & 7.26\% (24/26) & 5.66\% (19/26) \\
Full state & 8.91\% (31/34) & 9.34\% (30/34) & 8.20\% (24/34) \\
\bottomrule\end{tabular}
\end{table}

Member-mean RMSE decreases in 31/34 fields, including 23/26 fields not directly corrected. Across the latter, mean member-RMSE and fair-CRPS reductions are 5.97\% and 5.66\%.

Area-averaged local error decreases by 73.94\% for water and 50.54\% for the full state. The non-water gains are smaller and spatially heterogeneous, complementing the aggregate improvement reported in the main text.

\begin{figure}[!htbp]
\centering
\includegraphics[width=\linewidth]{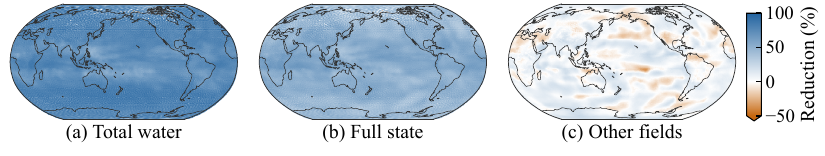}
\caption{Relative reduction in grouped spatial time-mean error for (a) 8 water layers, (b) all 34 fields, and (c) the 26 non-water fields. Blue indicates improvement and orange degradation.}
\label{fig:climate_group_tm_spatial}
\end{figure}

\FloatBarrier

\begingroup
\input{sections/climate_ablation_appendix.tex}

\endgroup

%% file: sections/climate_ablation_appendix.tex
\input{sections/appendix_e_ablations}

%% file: sections/appendix_e_ablations.tex

\subsubsection{Subsystem ablations}
\label{app:climate_ablation}

This section supplements Section~\ref{sec:climate_ablation} with absolute ST and TM errors and their relative changes from full CasEm for the ablations summarized in Table~\ref{tab:climate_ablation_main}, together with additional temporal and spatial comparisons.

\paragraph{Subsystem configurations.} Table~\ref{tab:climate_ablation_config} summarizes the subsystem configurations used in the horizontal-grid and vertical-group ablations; parameter counts exclude the backbone.

\begin{table}[!htbp]\centering\footnotesize
\caption{Target-model configurations for the horizontal-grid and vertical-group ablations.}
\label{tab:climate_ablation_config}
\begin{tabular}{@{}lrrrrr@{}}\toprule
Quantity & \multicolumn{5}{c}{Horizontal-grid configurations} \\\midrule
Horizontal grid & $1\!\times\!1$ & $3\!\times\!6$ & $6\!\times\!12$ & $12\!\times\!24$ & $18\!\times\!36$ \\
Parameters & 88 & 1,472 & 5,792 & 23,072 & 51,872 \\
Dynamic dimension & 8 & 32 & 32 & 32 & 32 \\
\bottomrule\end{tabular}
\par\vspace{0.4em}
\begin{tabular}{@{}lrrrr@{}}\toprule
Quantity & \multicolumn{4}{c}{Vertical-group configurations} \\\midrule
Vertical groups & 1 & 2 & 4 & 8 \\
Parameters & 2,884 & 5,768 & 11,536 & 23,072 \\
Dynamic dimension & 4 & 8 & 16 & 32 \\
\bottomrule\end{tabular}
\end{table}

\paragraph{Complete numerical results.} The relative-change and improvement-loss definitions are given in Appendix~\ref{app:climate_relative_changes}.

Table~\ref{tab:climate_ablation_absolute} gives the complete temporal results and additionally includes 10-year averaging, which replaces time-varying targets with a single 10-year-mean target.

\begin{table}[!htbp]\centering\footnotesize
\caption{Temporal target interventions: ST/TM NRMSE. Parentheses give relative changes (\%) from full CasEm; positive is worse.}
\label{tab:climate_ablation_absolute}
\setlength{\tabcolsep}{3pt}\renewcommand{\arraystretch}{0.94}
\begin{tabular}{@{}lrrrr@{}}\toprule
& \multicolumn{2}{c}{Full state} & \multicolumn{2}{c}{Total water} \\
\cmidrule(lr){2-3}\cmidrule(lr){4-5}
Configuration & ST & TM & ST & TM \\\midrule
Full CasEm & 0.80338 & 0.04011 & 0.73852 & 0.06652 \\
Monthly averaging & 0.80195 ($-0.2\%$) & 0.03830 ($-4.5\%$) & 0.73685 ($-0.2\%$) & 0.06521 ($-2.0\%$) \\
Quarterly averaging & 0.80326 ($0.0\%$) & 0.04171 ($+4.0\%$) & 0.73712 ($-0.2\%$) & 0.06475 ($-2.7\%$) \\
Annual averaging & 0.81774 ($+1.8\%$) & 0.04636 ($+15.6\%$) & 0.75383 ($+2.1\%$) & 0.07098 ($+6.7\%$) \\
10-year averaging & 0.81983 ($+2.0\%$) & 0.04576 ($+14.1\%$) & 0.76813 ($+4.0\%$) & 0.07205 ($+8.3\%$) \\
Seasonal phase shift & 0.81211 ($+1.1\%$) & 0.06260 ($+56.1\%$) & 0.76302 ($+3.3\%$) & 0.08093 ($+21.7\%$) \\
Year-1 target repeat & 0.83806 ($+4.3\%$) & 0.08031 ($+100.2\%$) & 0.87399 ($+18.3\%$) & 0.22821 ($+243.1\%$) \\
\bottomrule\end{tabular}
\end{table}

Table~\ref{tab:climate_retrained_grid_absolute} includes the additional $3\!\times\!6$ grid and the $12\!\times\!24$ reference. The $1\!\times\!1$ subsystem retains one arithmetic mean per water layer, giving a dynamic dimension of 8 rather than 32.

\begin{table}[!htbp]\centering\footnotesize
\caption{Retrained target grids: ST/TM NRMSE. Parentheses give relative changes (\%) from the $12\!\times\!24$ reference; positive is worse.}
\label{tab:climate_retrained_grid_absolute}
\setlength{\tabcolsep}{3pt}
\begin{tabular}{@{}lrrrr@{}}\toprule
& \multicolumn{2}{c}{Full state} & \multicolumn{2}{c}{Total water} \\
\cmidrule(lr){2-3}\cmidrule(lr){4-5}
Grid & ST & TM & ST & TM \\\midrule
$1\!\times\!1$ & 0.80909 ($+0.7\%$) & 0.06292 ($+56.9\%$) & 0.75808 ($+2.6\%$) & 0.13636 ($+105.0\%$) \\
$3\!\times\!6$ & 0.80427 ($+0.1\%$) & 0.04783 ($+19.2\%$) & 0.73847 ($0.0\%$) & 0.08588 ($+29.1\%$) \\
$6\!\times\!12$ & 0.80452 ($+0.1\%$) & 0.04631 ($+15.5\%$) & 0.73727 ($-0.2\%$) & 0.08053 ($+21.1\%$) \\
$12\!\times\!24$ & 0.80338 & 0.04011 & 0.73852 & 0.06652 \\
$18\!\times\!36$ & 0.80612 ($+0.3\%$) & 0.03994 ($-0.4\%$) & 0.74116 ($+0.4\%$) & 0.06669 ($+0.3\%$) \\
\bottomrule\end{tabular}
\end{table}

Table~\ref{tab:climate_vertical_retrained} extends vertical grouping to 1 and 2 groups and includes layer-mean redistribution. Groups use pressure-thickness-weighted layer means and supply shared mixing-ratio increments to their constituent layers. At each horizontal location, layer-mean redistribution assigns $1/8$ of the correction summed over the 8 layers to each layer before applying the correction constraints.

\begin{table}[!htbp]\centering\footnotesize
\caption{Vertical grouping and correction allocation: ST/TM NRMSE. Parentheses give relative changes (\%) from the 8-group reference; positive is worse.}
\label{tab:climate_vertical_retrained}
\setlength{\tabcolsep}{3pt}
\begin{tabular*}{\linewidth}{@{\extracolsep{\fill}}lrrrr@{}}\toprule
& \multicolumn{2}{c}{Full state} & \multicolumn{2}{c}{Total water} \\
\cmidrule(lr){2-3}\cmidrule(lr){4-5}
Groups & ST & TM & ST & TM \\\midrule
1 & \begin{tabular}[c]{@{}r@{}}32.19587\\($+3907.5\%$)\end{tabular} & \begin{tabular}[c]{@{}r@{}}8.59269\\($+21322.7\%$)\end{tabular} & \begin{tabular}[c]{@{}r@{}}82.51388\\($+11072.9\%$)\end{tabular} & \begin{tabular}[c]{@{}r@{}}19.14808\\($+28685.5\%$)\end{tabular} \\
2 & \begin{tabular}[c]{@{}r@{}}4.40199\\($+447.9\%$)\end{tabular} & \begin{tabular}[c]{@{}r@{}}1.31205\\($+3171.1\%$)\end{tabular} & \begin{tabular}[c]{@{}r@{}}11.10586\\($+1403.8\%$)\end{tabular} & \begin{tabular}[c]{@{}r@{}}2.78783\\($+4091.0\%$)\end{tabular} \\
4 & \begin{tabular}[c]{@{}r@{}}0.97867\\($+21.8\%$)\end{tabular} & \begin{tabular}[c]{@{}r@{}}0.24401\\($+508.3\%$)\end{tabular} & \begin{tabular}[c]{@{}r@{}}1.15386\\($+56.2\%$)\end{tabular} & \begin{tabular}[c]{@{}r@{}}0.64185\\($+864.9\%$)\end{tabular} \\
8 & 0.80338 & 0.04011 & 0.73852 & 0.06652 \\
\midrule
Layer mean & \begin{tabular}[c]{@{}r@{}}17.85571\\($+2122.6\%$)\end{tabular} & \begin{tabular}[c]{@{}r@{}}4.76563\\($+11781.3\%$)\end{tabular} & \begin{tabular}[c]{@{}r@{}}45.78716\\($+6099.9\%$)\end{tabular} & \begin{tabular}[c]{@{}r@{}}10.09095\\($+15069.8\%$)\end{tabular} \\
Fixed profile & \begin{tabular}[c]{@{}r@{}}5.76702\\($+617.8\%$)\end{tabular} & \begin{tabular}[c]{@{}r@{}}1.35919\\($+3288.6\%$)\end{tabular} & \begin{tabular}[c]{@{}r@{}}14.61720\\($+1879.3\%$)\end{tabular} & \begin{tabular}[c]{@{}r@{}}2.63062\\($+3854.6\%$)\end{tabular} \\
\bottomrule\end{tabular*}
\end{table}

\clearpage
\paragraph{Horizontal organization of the correction.}
The preceding ablations examine target resolution and
vertical correction allocation. Here, we further assess
whether effective guidance depends on the net correction
alone or also on its horizontal distribution.
Unlike target-grid retraining, these interventions transform
the correction field itself.
Each transformation preserves the signed area-weighted sum
within each water layer at every correction event, but not
necessarily the $L_2$ norm.
Longitude averaging retains the latitude profile, whereas
cosine-weighted latitude averaging retains the longitude
profile.

\begin{table}[!htbp]
\centering
\small
\caption{Absolute ST and TM NRMSE for horizontal
correction-field transformations.
Full CasEm is the unmodified reference; lower is better.
ST uses uniform spatial weights, whereas TM uses
cosine-latitude area weights.}
\label{tab:climate_horizontal_transforms}
\setlength{\tabcolsep}{5pt}
\begin{tabular}{@{}lrrrr@{}}
\toprule
& \multicolumn{2}{c}{Full state}
& \multicolumn{2}{c}{Total water} \\
\cmidrule(lr){2-3}\cmidrule(lr){4-5}
Transformation & ST & TM & ST & TM \\
\midrule
Full CasEm
& 0.80338 & 0.04011 & 0.73852 & 0.06652 \\
\midrule
Horizontally uniform
& 0.81495 & 0.05796 & 0.78065 & 0.11824 \\
Mean over longitude
& 0.80673 & 0.03870 & 0.74329 & 0.06419 \\
Cosine-weighted mean over latitude
& 0.81718 & 0.06227 & 0.77709 & 0.12481 \\
\bottomrule
\end{tabular}
\end{table}

Table~\ref{tab:climate_horizontal_transforms} shows that
horizontally uniform corrections increase both ST and TM
errors despite preserving the net correction within each
layer. The correction amount alone is therefore insufficient
to explain the gains.
Longitude averaging retains comparable performance,
including slightly lower TM errors, whereas averaging over
latitude substantially increases TM errors.
This contrast suggests that the latitude-dependent
distribution of the correction is more important than
detailed longitude dependence in this setting.
Together with the vertical-allocation results, these
findings support low-dimensional guidance that retains
key spatial organization without reproducing every
spatial detail.

%% file: main.bbl
\begin{thebibliography}{48}
\providecommand{\natexlab}[1]{#1}
\providecommand{\url}[1]{\texttt{#1}}
\expandafter\ifx\csname urlstyle\endcsname\relax
  \providecommand{\doi}[1]{doi: #1}\else
  \providecommand{\doi}{doi: \begingroup \urlstyle{rm}\Url}\fi

\bibitem[{\AA}str{\"o}m \& Murray(2008){\AA}str{\"o}m and Murray]{astrom2008feedback}
Karl~J. {\AA}str{\"o}m and Richard~M. Murray.
\newblock \emph{Feedback Systems: An Introduction for Scientists and Engineers}.
\newblock Princeton University Press, 2008.

\bibitem[Bassi et~al.(2026)Bassi, Zhu, Ye, Ren, Dektor, Mahoney, Bhat, and Yang]{bassi2026relift}
Hardeep Bassi, Yuanran Zhu, Erika Ye, Pu~Ren, Alec Dektor, Michael~W. Mahoney, Harish~S. Bhat, and Chao Yang.
\newblock {RELift}: Learned coarse-to-fine propagators for time-dependent {PDEs} with applications to plasma dynamics.
\newblock \emph{APL Computational Physics}, 2\penalty0 (2):\penalty0 026107, 2026.
\newblock \doi{10.1063/5.0333466}.

\bibitem[Beljaars(2003)]{beljaars2003hydrological}
A.~C.~M. Beljaars.
\newblock Some aspects of modelling of the hydrological cycle in the {ECMWF} model.
\newblock In \emph{{ECMWF/GEWEX} Workshop on Humidity Analysis, 8--11 July 2002}. ECMWF, 2003.
\newblock URL \url{https://www.ecmwf.int/sites/default/files/elibrary/2003/8039-some-aspects-modelling-hydrological-cycle-ecmwf-model.pdf}.

\bibitem[Bi et~al.(2023)Bi, Xie, Zhang, Chen, Gu, and Tian]{bi2023pangu}
Kaifeng Bi, Lingxi Xie, Hengheng Zhang, Xin Chen, Xiaotao Gu, and Qi~Tian.
\newblock Accurate medium-range global weather forecasting with {3D} neural networks.
\newblock \emph{Nature}, 619\penalty0 (7970):\penalty0 533--538, 2023.
\newblock \doi{10.1038/s41586-023-06185-3}.

\bibitem[Bonavita(2024)]{bonavita2024limitations}
Massimo Bonavita.
\newblock On some limitations of current machine learning weather prediction models.
\newblock \emph{Geophysical Research Letters}, 51\penalty0 (12):\penalty0 e2023GL107377, 2024.
\newblock \doi{10.1029/2023GL107377}.

\bibitem[Bonev et~al.(2023)Bonev, Kurth, Hundt, Pathak, Baust, Kashinath, and Anandkumar]{bonev2023sfno}
Boris Bonev, Thorsten Kurth, Christian Hundt, Jaideep Pathak, Maximilian Baust, Karthik Kashinath, and Anima Anandkumar.
\newblock Spherical {Fourier} neural operators: Learning stable dynamics on the sphere.
\newblock In \emph{Proceedings of the 40th International Conference on Machine Learning}, volume 202 of \emph{Proceedings of Machine Learning Research}, pp.\  2806--2823, 2023.
\newblock URL \url{https://proceedings.mlr.press/v202/bonev23a.html}.

\bibitem[Brandstetter et~al.(2022)Brandstetter, Worrall, and Welling]{brandstetter2022mppde}
Johannes Brandstetter, Daniel Worrall, and Max Welling.
\newblock Message passing neural {PDE} solvers.
\newblock In \emph{International Conference on Learning Representations}, 2022.

\bibitem[Champion et~al.(2019)Champion, Lusch, Kutz, and Brunton]{champion2019coordinates}
Kathleen Champion, Bethany Lusch, J.~Nathan Kutz, and Steven~L. Brunton.
\newblock Data-driven discovery of coordinates and governing equations.
\newblock \emph{Proceedings of the National Academy of Sciences}, 116\penalty0 (45):\penalty0 22445--22451, 2019.
\newblock \doi{10.1073/pnas.1906995116}.

\bibitem[Chattopadhyay et~al.(2024)Chattopadhyay, Sun, and Hassanzadeh]{chattopadhyay2024challenges}
Ashesh Chattopadhyay, Y.~Qiang Sun, and Pedram Hassanzadeh.
\newblock Challenges of learning multi-scale dynamics with {AI} weather models: Implications for stability and one solution.
\newblock \emph{arXiv preprint arXiv:2304.07029v2}, 2024.
\newblock URL \url{https://arxiv.org/abs/2304.07029v2}.

\bibitem[Chen et~al.(2023)Chen, Zhong, Zhang, Cheng, Xu, Qi, and Li]{chen2023fuxi}
Lei Chen, Xiaohui Zhong, Feng Zhang, Yuan Cheng, Yinghui Xu, Yuan Qi, and Hao Li.
\newblock {FuXi}: A cascade machine learning forecasting system for 15-day global weather forecast.
\newblock \emph{npj Climate and Atmospheric Science}, 6:\penalty0 190, 2023.
\newblock \doi{10.1038/s41612-023-00512-1}.

\bibitem[Eastman et~al.(2024)Eastman, Galvelis, Pel{\'a}ez, Abreu, Farr, Gallicchio, Gorenko, Henry, Hu, Huang, Kr{\"a}mer, Michel, Mitchell, Pande, Rodrigues, Rodriguez-Guerra, Simmonett, Singh, Swails, Turner, Wang, Zhang, Chodera, De~Fabritiis, and Markland]{eastman2024openmm}
Peter Eastman, Raimondas Galvelis, Ra{\'u}l~P. Pel{\'a}ez, Charlles R.~A. Abreu, Stephen~E. Farr, Emilio Gallicchio, Anton Gorenko, Michael~M. Henry, Frank Hu, Jing Huang, Andreas Kr{\"a}mer, Julien Michel, Joshua~A. Mitchell, Vijay~S. Pande, Jo{\~a}o P. G. L.~M. Rodrigues, Jaime Rodriguez-Guerra, Andrew~C. Simmonett, Sukrit Singh, Jason Swails, Philip Turner, Yuanqing Wang, Ivy Zhang, John~D. Chodera, Gianni De~Fabritiis, and Thomas~E. Markland.
\newblock {OpenMM 8}: Molecular dynamics simulation with machine learning potentials.
\newblock \emph{The Journal of Physical Chemistry B}, 128\penalty0 (1):\penalty0 109--116, 2024.
\newblock \doi{10.1021/acs.jpcb.3c06662}.

\bibitem[Goldstein(2011)]{goldstein2011classical}
Herbert Goldstein.
\newblock \emph{Classical mechanics}.
\newblock Pearson Education India, 2011.

\bibitem[Hu et~al.(2025)Hu, Wang, Zheng, Zhang, Feng, Feng, Wei, Wang, Ma, and Wu]{hu2024wdno}
Peiyan Hu, Rui Wang, Xiang Zheng, Tao Zhang, Haodong Feng, Ruiqi Feng, Long Wei, Yue Wang, Zhi-Ming Ma, and Tailin Wu.
\newblock Wavelet diffusion neural operator.
\newblock In \emph{International Conference on Learning Representations}, 2025.
\newblock URL \url{https://proceedings.iclr.cc/paper_files/paper/2025/hash/20eadb6338a53114a6371f102c5ded95-Abstract-Conference.html}.

\bibitem[Irvin et~al.(2026)Irvin, Han, Wang, Alharbi, Zhao, Bayarsaikhan, Visioni, Ng, and Watson-Parris]{irvin2026pyramid}
Jeremy~A. Irvin, Jiaqi Han, Zikui Wang, Abdulaziz Alharbi, Yufei Zhao, Nomin-Erdene Bayarsaikhan, Daniele Visioni, Andrew~Y. Ng, and Duncan Watson-Parris.
\newblock Spatiotemporal pyramid flow matching for climate emulation.
\newblock In \emph{Proceedings of the IEEE/CVF Conference on Computer Vision and Pattern Recognition}, 2026.

\bibitem[Jiang et~al.(2023)Jiang, Lu, Orlova, and Willett]{jiang2023invariant}
Ruoxi Jiang, Peter~Y. Lu, Elena Orlova, and Rebecca Willett.
\newblock Training neural operators to preserve invariant measures of chaotic attractors.
\newblock In \emph{Advances in Neural Information Processing Systems}, 2023.

\bibitem[Jiang et~al.(2025)Jiang, Zhang, Jakhar, Lu, Hassanzadeh, Maire, and Willett]{jiang2025hine}
Ruoxi Jiang, Xiao Zhang, Karan Jakhar, Peter~Y. Lu, Pedram Hassanzadeh, Michael Maire, and Rebecca Willett.
\newblock Hierarchical implicit neural emulators.
\newblock In \emph{Advances in Neural Information Processing Systems}, 2025.

\bibitem[Khalil(2002)]{khalil2002nonlinear}
Hassan~K. Khalil.
\newblock \emph{Nonlinear Systems}.
\newblock Prentice Hall, 3rd edition, 2002.

\bibitem[Kl{\"o}wer et~al.(2024)Kl{\"o}wer, Gelbrecht, Hotta, Willmert, Silvestri, Wagner, White, Hatfield, Kimpson, Constantinou, and Hill]{klower2024speedyweather}
Milan Kl{\"o}wer, Maximilian Gelbrecht, Daisuke Hotta, Justin Willmert, Simone Silvestri, Gregory~L. Wagner, Alistair White, Sam Hatfield, Tom Kimpson, Navid~C. Constantinou, and Chris Hill.
\newblock {SpeedyWeather.jl}: Reinventing atmospheric general circulation models towards interactivity and extensibility.
\newblock \emph{Journal of Open Source Software}, 9\penalty0 (98):\penalty0 6323, 2024.
\newblock \doi{10.21105/joss.06323}.

\bibitem[Lam et~al.(2023)Lam, Sanchez-Gonzalez, Willson, Wirnsberger, Fortunato, Alet, Ravuri, Ewalds, Eaton-Rosen, Hu, et~al.]{lam2023graphcast}
Remi Lam, Alvaro Sanchez-Gonzalez, Matthew Willson, Peter Wirnsberger, Meire Fortunato, Ferran Alet, Suman Ravuri, Timo Ewalds, Zach Eaton-Rosen, Weihua Hu, et~al.
\newblock Learning skillful medium-range global weather forecasting.
\newblock \emph{Science}, 382\penalty0 (6677):\penalty0 1416--1421, 2023.
\newblock \doi{10.1126/science.adi2336}.

\bibitem[Lee et~al.(2024)Lee, Gleckler, Ahn, Ordonez, Ullrich, Sperber, Taylor, Planton, Guilyardi, Durack, Bonfils, Zelinka, Chao, Dong, Doutriaux, Zhang, Vo, Boutte, Wehner, Pendergrass, Kim, Xue, Wittenberg, and Krasting]{lee2024pmp}
J.~Lee, P.~J. Gleckler, M.-S. Ahn, A.~Ordonez, P.~A. Ullrich, K.~R. Sperber, K.~E. Taylor, Y.~Y. Planton, E.~Guilyardi, P.~Durack, C.~Bonfils, M.~D. Zelinka, L.-W. Chao, B.~Dong, C.~Doutriaux, C.~Zhang, T.~Vo, J.~Boutte, M.~F. Wehner, A.~G. Pendergrass, D.~Kim, Z.~Xue, A.~T. Wittenberg, and J.~Krasting.
\newblock Systematic and objective evaluation of earth system models: {PCMDI Metrics Package (PMP)} version 3.
\newblock \emph{Geoscientific Model Development}, 17:\penalty0 3919--3948, 2024.
\newblock \doi{10.5194/gmd-17-3919-2024}.

\bibitem[Li et~al.(2021)Li, Kovachki, Azizzadenesheli, Liu, Bhattacharya, Stuart, and Anandkumar]{li2021fno}
Zongyi Li, Nikola Kovachki, Kamyar Azizzadenesheli, Burigede Liu, Kaushik Bhattacharya, Andrew Stuart, and Anima Anandkumar.
\newblock {Fourier} neural operator for parametric partial differential equations.
\newblock In \emph{International Conference on Learning Representations}, 2021.

\bibitem[Li et~al.(2022)Li, Liu-Schiaffini, Kovachki, Azizzadenesheli, Liu, Bhattacharya, Stuart, and Anandkumar]{li2022chaotic}
Zongyi Li, Miguel Liu-Schiaffini, Nikola Kovachki, Kamyar Azizzadenesheli, Burigede Liu, Kaushik Bhattacharya, Andrew Stuart, and Anima Anandkumar.
\newblock Learning chaotic dynamics in dissipative systems.
\newblock In \emph{Advances in Neural Information Processing Systems}, 2022.

\bibitem[Linot \& Graham(2022)Linot and Graham]{linot2022reduced}
Alec~J. Linot and Michael~D. Graham.
\newblock Data-driven reduced-order modeling of spatiotemporal chaos with neural ordinary differential equations.
\newblock \emph{Chaos}, 32\penalty0 (7):\penalty0 073110, 2022.
\newblock \doi{10.1063/5.0069536}.

\bibitem[Lippe et~al.(2023)Lippe, Veeling, Perdikaris, Turner, and Brandstetter]{lippe2023pderefiner}
Phillip Lippe, Bastiaan Veeling, Paris Perdikaris, Richard~E. Turner, and Johannes Brandstetter.
\newblock {PDE-Refiner}: Achieving accurate long rollouts with neural {PDE} solvers.
\newblock In \emph{Advances in Neural Information Processing Systems}, 2023.

\bibitem[List et~al.(2025)List, Chen, Bali, and Thuerey]{list2025differentiability}
Bjoern List, Li-Wei Chen, Kartik Bali, and Nils Thuerey.
\newblock Differentiability in unrolled training of neural physics simulators on transient dynamics.
\newblock \emph{Computer Methods in Applied Mechanics and Engineering}, 433:\penalty0 117441, 2025.
\newblock \doi{10.1016/j.cma.2024.117441}.

\bibitem[Lupin-Jimenez et~al.(2025)Lupin-Jimenez, Darman, Hazarika, Wu, Gray, He, Wong, and Chattopadhyay]{lupinjimenez2025fcds}
Leonard Lupin-Jimenez, Moein Darman, Subhashis Hazarika, Tianning Wu, Michael Gray, Ruoying He, Anthony Wong, and Ashesh Chattopadhyay.
\newblock Simultaneous emulation and downscaling with physically consistent deep learning-based regional ocean emulators.
\newblock \emph{Journal of Geophysical Research: Machine Learning and Computation}, 2\penalty0 (3):\penalty0 e2025JH000851, 2025.
\newblock \doi{10.1029/2025JH000851}.

\bibitem[Lusch et~al.(2018)Lusch, Kutz, and Brunton]{lusch2018koopman}
Bethany Lusch, J.~Nathan Kutz, and Steven~L. Brunton.
\newblock Deep learning for universal linear embeddings of nonlinear dynamics.
\newblock \emph{Nature Communications}, 9:\penalty0 4950, 2018.
\newblock \doi{10.1038/s41467-018-07210-0}.

\bibitem[Lyons(2004)]{lyons2004understanding}
Richard~G. Lyons.
\newblock \emph{Understanding Digital Signal Processing}.
\newblock Prentice Hall, 2nd edition, 2004.
\newblock ISBN 9780131089891.

\bibitem[Maier et~al.(2015)Maier, Martinez, Kasavajhala, Wickstrom, Hauser, and Simmerling]{maier2015ff14sb}
James~A. Maier, Carmenza Martinez, Koushik Kasavajhala, Lauren Wickstrom, Kevin~E. Hauser, and Carlos Simmerling.
\newblock {ff14SB}: Improving the accuracy of protein side chain and backbone parameters from {ff99SB}.
\newblock \emph{Journal of Chemical Theory and Computation}, 11\penalty0 (8):\penalty0 3696--3713, 2015.
\newblock \doi{10.1021/acs.jctc.5b00255}.

\bibitem[Morel et~al.(2025)Morel, Ramunno, Shen, Bietti, Cho, Cranmer, Golkar, Gugnin, Krawezik, Marwah, McCabe, Meyer, Mukhopadhyay, Ohana, Parker, Qu, Rozet, Leka, Lanusse, Fouhey, and Ho]{morel2025multiscale}
Rudy Morel, Francesco~Pio Ramunno, Jeff Shen, Alberto Bietti, Kyunghyun Cho, Miles Cranmer, Siavash Golkar, Olexandr Gugnin, Geraud Krawezik, Tanya Marwah, Michael McCabe, Lucas Meyer, Payel Mukhopadhyay, Ruben Ohana, Liam Parker, Helen Qu, Fran{\c c}ois Rozet, K.~D. Leka, Fran{\c c}ois Lanusse, David Fouhey, and Shirley Ho.
\newblock Predicting partially observable dynamical systems via diffusion models with a multiscale inference scheme.
\newblock In \emph{Advances in Neural Information Processing Systems}, volume~38, pp.\  74387--74405, 2025.
\newblock \doi{10.52202/085713-2240}.

\bibitem[Musaelian et~al.(2023)Musaelian, Batzner, Johansson, Sun, Owen, Kornbluth, and Kozinsky]{musaelian2023allegro}
Albert Musaelian, Simon Batzner, Anders Johansson, Lixin Sun, Cameron~J. Owen, Mordechai Kornbluth, and Boris Kozinsky.
\newblock Learning local equivariant representations for large-scale atomistic dynamics.
\newblock \emph{Nature Communications}, 14:\penalty0 579, 2023.
\newblock \doi{10.1038/s41467-023-36329-y}.

\bibitem[Onufriev et~al.(2004)Onufriev, Bashford, and Case]{onufriev2004obc}
Alexey Onufriev, Donald Bashford, and David~A. Case.
\newblock Exploring protein native states and large-scale conformational changes with a modified generalized {Born} model.
\newblock \emph{Proteins: Structure, Function, and Bioinformatics}, 55\penalty0 (2):\penalty0 383--394, 2004.
\newblock \doi{10.1002/prot.20033}.

\bibitem[Pathak et~al.(2017)Pathak, Ghosh, Kumar, and Murtugudde]{pathak2017moisture}
Amey Pathak, Subimal Ghosh, Praveen Kumar, and Raghu Murtugudde.
\newblock Role of oceanic and terrestrial atmospheric moisture sources in intraseasonal variability of {Indian} summer monsoon rainfall.
\newblock \emph{Scientific Reports}, 7:\penalty0 12729, 2017.
\newblock \doi{10.1038/s41598-017-13115-7}.

\bibitem[Pathak et~al.(2022)Pathak, Subramanian, Harrington, Raja, Chattopadhyay, Mardani, Kurth, Hall, Li, Azizzadenesheli, et~al.]{pathak2022fourcastnet}
Jaideep Pathak, Shashank Subramanian, Peter Harrington, Sanjeev Raja, Ashesh Chattopadhyay, Morteza Mardani, Thorsten Kurth, David Hall, Zongyi Li, Kamyar Azizzadenesheli, et~al.
\newblock {FourCastNet}: A global data-driven high-resolution weather model using adaptive {Fourier} neural operators.
\newblock \emph{arXiv preprint arXiv:2202.11214}, 2022.

\bibitem[Perkins et~al.(2025)Perkins, Kwa, McGibbon, Arcomano, Clark, Watt-Meyer, Bretherton, and Harris]{perkins2025hiroace}
W.~Andre Perkins, Anna Kwa, Jeremy McGibbon, Troy Arcomano, Spencer~K. Clark, Oliver Watt-Meyer, Christopher~S. Bretherton, and Lucas~M. Harris.
\newblock {HiRO-ACE}: Fast and skillful {AI} emulation and downscaling trained on a 3 km global storm-resolving model.
\newblock \emph{arXiv preprint arXiv:2512.18224}, 2025.
\newblock URL \url{https://arxiv.org/abs/2512.18224}.

\bibitem[Price et~al.(2025)Price, Sanchez-Gonzalez, Alet, Andersson, El-Kadi, Masters, Ewalds, Stott, Mohamed, Battaglia, Lam, and Willson]{price2025gencast}
Ilan Price, Alvaro Sanchez-Gonzalez, Ferran Alet, Tom~R. Andersson, Andrew El-Kadi, Dominic Masters, Timo Ewalds, Jacklynn Stott, Shakir Mohamed, Peter Battaglia, Remi Lam, and Matthew Willson.
\newblock Probabilistic weather forecasting with machine learning.
\newblock \emph{Nature}, 637\penalty0 (8044):\penalty0 84--90, 2025.
\newblock \doi{10.1038/s41586-024-08252-9}.

\bibitem[Rozet et~al.(2025)Rozet, Ohana, McCabe, Louppe, Lanusse, and Ho]{rozet2025latent}
Fran{\c c}ois Rozet, Ruben Ohana, Michael McCabe, Gilles Louppe, Fran{\c c}ois Lanusse, and Shirley Ho.
\newblock Lost in latent space: An empirical study of latent diffusion models for physics emulation.
\newblock In \emph{Advances in Neural Information Processing Systems}, 2025.

\bibitem[R{\"u}hling~Cachay et~al.(2023)R{\"u}hling~Cachay, Zhao, Joren, and Yu]{ruhlingcachay2023dyffusion}
Salva R{\"u}hling~Cachay, Bo~Zhao, Hailey Joren, and Rose Yu.
\newblock {DYffusion}: A dynamics-informed diffusion model for spatiotemporal forecasting.
\newblock In \emph{Advances in Neural Information Processing Systems}, 2023.

\bibitem[R{\"u}hling~Cachay et~al.(2024)R{\"u}hling~Cachay, Henn, Watt-Meyer, Bretherton, and Yu]{ruhlingcachay2024sphericaldyffusion}
Salva R{\"u}hling~Cachay, Brian Henn, Oliver Watt-Meyer, Christopher~S. Bretherton, and Rose Yu.
\newblock Probabilistic emulation of a global climate model with spherical {DYffusion}.
\newblock In \emph{Advances in Neural Information Processing Systems}, 2024.

\bibitem[Sanchez-Gonzalez et~al.(2020)Sanchez-Gonzalez, Godwin, Pfaff, Ying, Leskovec, and Battaglia]{sanchezgonzalez2020gns}
Alvaro Sanchez-Gonzalez, Jonathan Godwin, Tobias Pfaff, Rex Ying, Jure Leskovec, and Peter Battaglia.
\newblock Learning to simulate complex physics with graph networks.
\newblock In \emph{Proceedings of the 37th International Conference on Machine Learning}, volume 119 of \emph{Proceedings of Machine Learning Research}, pp.\  8459--8468, 2020.
\newblock URL \url{https://proceedings.mlr.press/v119/sanchez-gonzalez20a.html}.

\bibitem[Satorras et~al.(2021)Satorras, Hoogeboom, and Welling]{satorras2021egnn}
Victor~Garcia Satorras, Emiel Hoogeboom, and Max Welling.
\newblock {$E(n)$} equivariant graph neural networks.
\newblock In \emph{Proceedings of the 38th International Conference on Machine Learning}, volume 139 of \emph{Proceedings of Machine Learning Research}, pp.\  9323--9332, 2021.
\newblock URL \url{https://proceedings.mlr.press/v139/satorras21a.html}.

\bibitem[Trefethen(2000)]{trefethen2000spectral}
Lloyd~N Trefethen.
\newblock \emph{Spectral methods in MATLAB}.
\newblock SIAM, 2000.

\bibitem[Tripura \& Chakraborty(2023)Tripura and Chakraborty]{tripura2022wno}
Tapas Tripura and Souvik Chakraborty.
\newblock Wavelet neural operator for solving parametric partial differential equations in computational mechanics problems.
\newblock \emph{Computer Methods in Applied Mechanics and Engineering}, 404:\penalty0 115783, 2023.
\newblock \doi{10.1016/j.cma.2022.115783}.

\bibitem[Vlachas et~al.(2022)Vlachas, Arampatzis, Uhler, and Koumoutsakos]{vlachas2022led}
Pantelis~R. Vlachas, Georgios Arampatzis, Caroline Uhler, and Petros Koumoutsakos.
\newblock Multiscale simulations of complex systems by learning their effective dynamics.
\newblock \emph{Nature Machine Intelligence}, 4:\penalty0 359--366, 2022.
\newblock \doi{10.1038/s42256-022-00464-w}.

\bibitem[Watt-Meyer et~al.(2023)Watt-Meyer, Dresdner, McGibbon, Clark, Henn, Duncan, Brenowitz, Kashinath, Pritchard, Bonev, Peters, and Bretherton]{wattmeyer2023ace}
Oliver Watt-Meyer, Gideon Dresdner, Jeremy McGibbon, Spencer~K. Clark, Brian Henn, James Duncan, Noah~D. Brenowitz, Karthik Kashinath, Michael~S. Pritchard, Boris Bonev, Matthew~E. Peters, and Christopher~S. Bretherton.
\newblock {ACE}: A fast, skillful learned global atmospheric model for climate prediction.
\newblock \emph{arXiv preprint arXiv:2310.02074}, 2023.
\newblock URL \url{https://arxiv.org/abs/2310.02074}.

\bibitem[Watt-Meyer et~al.(2025)Watt-Meyer, Henn, McGibbon, Clark, Kwa, Perkins, Wu, Harris, and Bretherton]{wattmeyer2025ace2}
Oliver Watt-Meyer, Brian Henn, Jeremy McGibbon, Spencer~K. Clark, Anna Kwa, W.~Andre Perkins, Elynn Wu, Lucas Harris, and Christopher~S. Bretherton.
\newblock {ACE2}: Accurately learning subseasonal to decadal atmospheric variability and forced responses.
\newblock \emph{npj Climate and Atmospheric Science}, 8:\penalty0 205, 2025.
\newblock \doi{10.1038/s41612-025-01090-0}.

\bibitem[Webster et~al.(1998)Webster, Maga{\~n}a, Palmer, Shukla, Tomas, Yanai, and Yasunari]{webster1998monsoons}
P.~J. Webster, V.~O. Maga{\~n}a, T.~N. Palmer, J.~Shukla, R.~A. Tomas, M.~Yanai, and T.~Yasunari.
\newblock Monsoons: Processes, predictability, and the prospects for prediction.
\newblock \emph{Journal of Geophysical Research: Oceans}, 103\penalty0 (C7):\penalty0 14451--14510, 1998.
\newblock \doi{10.1029/97JC02719}.

\bibitem[Zhou et~al.(2019)Zhou, Lin, Chen, Harris, Chen, and Rees]{zhou2019fv3gfs}
Linjiong Zhou, Shian-Jiann Lin, Jan-Huey Chen, Lucas~M. Harris, Xi~Chen, and Shannon~L. Rees.
\newblock Toward convective-scale prediction within the next generation global prediction system.
\newblock \emph{Bulletin of the American Meteorological Society}, 100\penalty0 (7):\penalty0 1225--1243, 2019.
\newblock \doi{10.1175/BAMS-D-17-0246.1}.

\end{thebibliography}
